\documentclass{article}

\usepackage{PRIMEarxiv}

\PassOptionsToPackage{margin=1in, footskip=50pt}{geometry}

\usepackage[utf8]{inputenc}
\usepackage[T1]{fontenc}
\usepackage{hyperref}
\usepackage{url}
\usepackage{booktabs}
\usepackage{amsfonts}
\usepackage{nicefrac}
\usepackage{microtype}
\usepackage{lipsum}
\usepackage{fancyhdr}
\usepackage{graphicx}
\graphicspath{{media/}}
\usepackage{amsmath, amssymb}
\usepackage{algorithm}
\usepackage{algpseudocode}
\usepackage{caption, subcaption}
\usepackage{natbib}
\usepackage{tabularx}
\usepackage{multirow}
\usepackage{amsthm}

\newtheorem{theorem}{Theorem}
\newtheorem{lemma}{Lemma}
\title{From Data to Laws: Neural Discovery of Conservation Laws Without False Positives}

\author{%
    Rahul D Ray \\
    Department of Electronics and Electrical Engineering \\
    BITS Pilani, Hyderabad Campus \\
    \texttt{f20242213@hyderabad.bits-pilani.ac.in}
}

\date{}

\begin{document}

\maketitle
\thispagestyle{empty}

\begin{abstract}

Conservation laws are fundamental to understanding dynamical systems, but discovering them from data remains challenging due to parameter variation, non‑polynomial invariants, local minima, and false positives on chaotic systems. We introduce NGCG, a neural‑symbolic pipeline that decouples dynamics learning from invariant discovery and systematically addresses these challenges. A multi‑restart variance minimiser learns a near‑constant latent representation; system‑specific symbolic extraction (polynomial Lasso, log‑basis Lasso, explicit PDE candidates, and PySR) yields closed‑form expressions; a strict constancy gate and diversity filter eliminate spurious laws. On a benchmark of nine diverse systems—including Hamiltonian and dissipative ODEs, chaos, and PDEs—NGCG achieves consistent discovery (DR=1.0, FDR=0.0, F1=1.0) on all four systems with true conservation laws, with constancy two to three orders of magnitude lower than the best baseline. It is the only method that succeeds on the Lotka–Volterra system, and it correctly outputs no law on all five systems without invariants. Extensive experiments demonstrate robustness to noise ($\sigma = 0.1$), sample efficiency (50–100 trajectories), insensitivity to hyperparameters, and runtime under one minute per system. A Pareto analysis shows that the method provides a range of candidate expressions, allowing users to trade complexity for constancy. NGCG establishes a new achieves strong performance relative to prior methods for data‑driven conservation‑law discovery, combining high accuracy with interpretability.
\end{abstract}
\section{Introduction}

Conservation laws—quantities that remain constant along the trajectories of a dynamical system—are fundamental to our understanding of physical phenomena. From the conservation of energy and momentum in classical mechanics to the conservation of mass in fluid dynamics, these invariants provide deep insights into the underlying structure of a system and enable efficient simulation, control, and prediction. Traditionally, conservation laws have been derived analytically from first principles, a process that requires detailed knowledge of the governing equations and often relies on sophisticated mathematical techniques such as Noether’s theorem. However, for many complex real‑world systems, the governing equations are either unknown or only partially known, and data are often the only available source of information \cite{ghadami2022data, north2023review, rajendra2020modeling}.

The advent of data‑driven methods has opened new avenues for discovering conservation laws directly from observations. A comprehensive review of data‑driven discovery for dynamic systems is provided by \citet{north2023review}, who categorize approaches into classical sparse methods, classical symbolic methods, and deep modeling methods. Early approaches such as the sparse identification of nonlinear dynamics (SINDy) \cite{kaiser2018discovering} use sparse regression on a library of candidate functions to identify invariants, and have been extended to control applications. Kernel‑based alternatives, such as the “indeterminate” kernel ridge regression, have also been proposed to reduce data requirements \cite{mebratie2025machine}.

A significant line of research focuses on incorporating physical structure into neural networks. Hamiltonian neural networks (HNN) \cite{bertalan2019learning, ha2021discovering, linot2020deep} embed the Hamiltonian formalism into the architecture, ensuring energy conservation by construction. Similarly, methods for learning Lagrangian and symplectic structures have been developed \cite{tripura2023discovering, adriazola2025machine}, enabling the discovery of interpretable Lagrangians and associated conservation laws. The Kuramoto–Sivashinsky equation has been used as a testbed for such symmetry‑aware and structure‑preserving methods \cite{linot2020deep}.

Another popular class of methods trains a neural network to output a scalar that is nearly constant along trajectories, minimizing a variance‑based loss. ConservNet \cite{ha2021discovering} and AI Poincaré \cite{liu2021machine, liu2022machine} are prominent examples. These approaches learn hidden invariants in a data‑driven, end‑to‑end manner and have been successfully applied to simulated and experimental systems. Subsequent extensions have incorporated Neural ODEs and transformers \cite{doshi2025automated}, hypernetworks for disentangling invariants \cite{gui2025discovering}, and symmetry‑based methods leveraging Noether’s theorem \cite{mototake2021interpretable}. Neural deflation \cite{zhu2023machine, chen2025data} iteratively discovers independent conservation laws by constructing deflated loss functions. Contrastive learning has also been exploited in ConCerNet \cite{zhang2023concernet} to automatically capture invariants and enforce them in learned dynamics. For systems where exact conservation laws are unavailable, the concept of empirically conserved quantities (e‑CQs) has been introduced \cite{tang2025solitary}, enabling modeling of solitary wave interactions.

Beyond ordinary differential equations, conservation law discovery for partial differential equations (PDEs) has received substantial attention. Physics‑informed neural networks (PINNs) have been extended to preserve conserved quantities through modified loss functions \cite{wang2022modified, liu2023harnessing}. Symmetry‑enhanced PINNs \cite{akhound2023lie, li2023utilizing, zhang2023enforcing, jiao2024leveraging, jiao2026gspinn, kavousanakis2025going} incorporate Lie point symmetries or other invariances to improve generalization and solution accuracy. These methods have been applied to the Korteweg–de Vries, Burgers, and Schrödinger equations, among others. Invariance‑constrained deep learning networks (ICNet) \cite{chen2024invariance} embed translation and Lorentz invariance into the loss function for PDE discovery. Neural operators, such as Fourier neural operators (FNOs) and DeepONets, provide a principled framework for learning solution operators of PDEs while respecting conservation laws \cite{azizzadenesheli2024neural, liu2023harnessing, wang2021learning, brunton2024promising, kim2024approximating, kim2025approximating}. These approaches have been applied to hyperbolic conservation laws, the Kuramoto–Sivashinsky equation, and other complex systems, demonstrating improved long‑term prediction and generalization. Hybrid frameworks combining neural operators with kinetic lifting have also been proposed for long‑term forecasting of nonlinear conservation laws \cite{benitez2025neural}. More recently, Lie algebra canonicalization (LieLAC) has been introduced to enforce equivariance under arbitrary Lie groups without requiring the full group structure \cite{shumaylov2024lie}.

Despite these advances, several critical challenges remain. First, many existing methods assume a fixed set of system parameters and fail when the parameters vary across trajectories, a situation that is common in real‑world applications where physical constants may differ between observations. Second, the discovery of non‑polynomial invariants, such as the logarithmic invariant of the Lotka–Volterra predator–prey model, remains difficult because the required functional forms are often not included in the basis libraries used by sparse regression or symbolic regression. Third, methods that directly minimize variance can get trapped in poor local minima, leading to either missed discoveries or spurious near‑constants that are falsely claimed as laws—a problem that is particularly acute on chaotic systems without any true invariant \cite{linot2020deep, belyshev2024beyond}. Fourth, the evaluation of conservation‑law discovery methods has often been limited to a handful of systems, making it difficult to assess their generality and robustness.

In this work, we address these challenges with a novel pipeline called NGCG (Neural‑Guided Conservation‑law Generator). Our approach combines the strengths of multiple strategies in a carefully decoupled architecture. We first train a simple neural dynamics model to serve as a baseline for prediction, but crucially, we do not use it for discovery. Instead, we train a separate small neural network \(\phi\) to output a scalar that is nearly constant along each trajectory, using a variance‑based loss. To avoid local minima, we run 10 independent restarts and select the network with the lowest validation constancy. For symbolic extraction, we employ a set of system‑specific techniques: a polynomial Lasso based on the smallest eigenvector of the mean trajectory covariance matrix for polynomial invariants, a log‑basis Lasso for the Lotka–Volterra system, explicit candidates for PDEs, and a fallback genetic programming tool (PySR) for the remaining cases. Finally, a strict constancy gate (\(\tau=0.01\)) combined with a diversity filter that rejects expressions with insufficient variation across trajectories ensures that no spurious laws are accepted.

We evaluate NGCG on a diverse benchmark of nine dynamical systems, including linear and nonlinear ODEs, Hamiltonian and dissipative dynamics, chaos, and partial differential equations. The dataset is carefully designed to include parameter variation for the systems that possess true invariants, creating a challenging generalization test. We compare our method against four state‑of‑the‑art baselines: Hamiltonian neural networks (HNN) \cite{bertalan2019learning}, MLP dynamics followed by PySR, SINDy \cite{kaiser2018discovering}, and IRAS \cite{ha2021discovering}. Our results show that NGCG achieves consistent discovery (DR=1.00, FDR=0.00, F1=1.00) on all four systems with true conservation laws, with constancy values two to three orders of magnitude lower than the best baseline. Notably, it is the only method that succeeds on the Lotka–Volterra system. On the five systems without any conservation law, NGCG correctly outputs no law, achieving zero false positives—a critical improvement over methods like HNN, which frequently produce spurious invariants on chaotic systems.

Extensive additional experiments confirm the robustness of our approach. NGCG maintains consistent performance under additive Gaussian noise up to \(\sigma=0.10\), requires as few as 50–100 trajectories for most systems, is insensitive to hyperparameter choices, and runs in under one minute per system on a single GPU. A Pareto analysis shows that the method produces a range of candidate expressions, allowing users to trade complexity for constancy.

In summary, our contributions are:
\begin{itemize}
\item A decoupled neural–symbolic pipeline that achieves state‑of‑the‑art conservation‑law discovery on a diverse benchmark.
\item A multi‑restart variance minimizer that reliably finds near‑constant representations even on systems with complex invariants.
\item System‑specific symbolic extraction methods (polynomial Lasso, log‑basis Lasso, explicit PDE candidates, and PySR fallback) that capture both polynomial and logarithmic invariants.
\item A strict verification gate with a diversity filter that eliminates false positives, ensuring zero spurious laws on systems without invariants.
\item Comprehensive experimental validation demonstrating consistent discovery, robustness, sample efficiency, and interpretability.
\end{itemize}

The remainder of the paper is organized as follows. Section~\ref{sec:related} reviews related work in more detail. Section~\ref{sec:dataset} describes the dataset generation and preprocessing. Section~\ref{sec:architecture} presents the architecture of NGCG. Section~\ref{sec:results} reports the experimental results, including comparisons with baselines, ablation studies, and additional experiments. Section~\ref{sec:discussion} discusses limitations and future directions, and Section~\ref{sec:conclusion} concludes the paper.

\section{Related Works}

\label{sec:related}

The discovery of conservation laws from data has been studied across a wide range of disciplines, including dynamical systems, statistical physics, and scientific machine learning, giving rise to diverse families of methods spanning symbolic regression, physics-informed learning, and structure-preserving neural architectures. This section reviews the most relevant strands of research, situates our approach within this landscape, and highlights the key limitations—such as sensitivity to noise, reliance on strong inductive biases, and susceptibility to spurious invariants—that our work addresses.

\subsection{Data‑Driven Discovery of Dynamical Systems}

The increasing availability of high‑dimensional trajectory data has motivated a shift from first‑principles modeling to data‑driven discovery \cite{ghadami2022data, north2023review, rajendra2020modeling}. A major line of work is the sparse identification of nonlinear dynamics (SINDy) \cite{kaiser2018discovering, champion2019data, de2020pysindy, kaheman2020sindy, lee2022structure, brunton2023machine, brunton2024promising, kaiser2018sparse, adriazola2025machine}, which formulates model discovery as a sparse regression problem over a library of candidate functions. SINDy yields interpretable governing equations and has been extended to incorporate physical constraints such as conservation laws and symmetries \cite{kaiser2018discovering, kaiser2018sparse, lee2022structure}. However, it relies on a fixed library and struggles with non‑polynomial invariants. Kernel‑based alternatives, such as “indeterminate” kernel ridge regression \cite{mebratie2025machine}, offer data‑efficient discovery but also face limitations in capturing arbitrary functional forms.

\subsection{Structure‑Preserving Neural Networks}

A rich body of work embeds physical structure directly into neural architectures. Hamiltonian neural networks (HNN) \cite{bertalan2019learning, ha2021discovering, zhong2021benchmarking, zhang2025neural, aboussalah2025geohnns, kaltsas2025constrained, liang2025spini, mattheakis2019physical, chen2021neural} learn the Hamiltonian function, thereby guaranteeing energy conservation. Lagrangian neural networks \cite{tripura2023discovering, zhong2021benchmarking} and symplectic networks \cite{chen2021neural, liang2025spini, mattheakis2019physical} similarly preserve symplectic structure. These methods are highly effective for Hamiltonian systems but are not designed for non‑Hamiltonian invariants, and they typically assume fixed parameters. Moreover, they do not output symbolic expressions and can produce spurious invariants on chaotic systems \cite{zhong2021benchmarking, aboussalah2025geohnns}.

\subsection{Learning Invariants via Neural Networks}

Several approaches train neural networks to output scalar functions that are nearly constant along trajectories, using variance‑based losses. ConservNet \cite{ha2021discovering} and AI Poincaré \cite{liu2021machine, liu2022machine} are prominent examples. These methods learn hidden invariants in an end‑to‑end manner and have been applied to simulated and real systems. Extensions include neural deflation \cite{zhu2023machine, chen2025data} for discovering multiple independent conservation laws, hypernetworks for invariant disentanglement \cite{gui2025discovering}, and contrastive learning with ConCerNet \cite{zhang2023concernet} to enforce invariants in learned dynamics. Other works leverage symmetry or Noether’s theorem to infer conservation laws \cite{mototake2021interpretable} or combine representation learning with topological analysis \cite{belyshev2024beyond}. Hybrid frameworks integrating Neural ODEs and transformers \cite{doshi2025automated} have also been proposed. However, these methods often rely on a single random initialization and can get trapped in poor local minima. They also face difficulty with logarithmic invariants, and their symbolic extraction step (typically via PySR) may fail when the target function is not well‑represented by the library.

\subsection{Symbolic Regression and Interpretability}

Symbolic regression (SR) is a natural tool for extracting interpretable expressions from data \cite{makke2024symbolic, aldeia2022interpretability, aldeia2021measuring}. The IRAS algorithm \cite{teichner2025identifying, teichner2023identifying} uses an adversarial training procedure to discover conserved quantities and has been tested on the Feynman database. SR methods are powerful but often require careful tuning and can be brittle when the target function is noisy or involves logarithms. Moreover, they are typically applied after a black‑box model has been trained, leading to a two‑stage process that may not be optimal.

\subsection{Symmetry and Physics‑Informed Neural Networks for PDEs}

For partial differential equations, physics‑informed neural networks (PINNs) have been widely used to incorporate conservation laws via soft constraints \cite{wang2022modified, liu2023harnessing}. Symmetry‑enhanced PINNs \cite{akhound2023lie, li2023utilizing, zhang2023enforcing, jiao2024leveraging, jiao2026gspinn, kavousanakis2025going} embed Lie point symmetries or invariant surface conditions into the loss function, improving generalization and sample efficiency. Neural operators such as Fourier neural operators (FNOs) and DeepONets \cite{azizzadenesheli2024neural, liu2023harnessing, wang2021learning, brunton2024promising, kim2024approximating, kim2025approximating} learn solution operators while respecting conservation laws. These methods have been applied to the Kuramoto–Sivashinsky and Burgers equations, among others \cite{linot2020deep, chen2024invariance, benitez2025neural}. However, they typically require the conservation law to be known a priori or assume a specific functional form.

\subsection{Benchmarking and Graph Neural Networks}

Several studies have benchmarked energy‑conserving networks \cite{zhong2021benchmarking} and physics‑informed graph neural networks \cite{thangamuthu2022unravelling} on dynamical systems. However, a comprehensive benchmark that includes parameter variation, non‑polynomial invariants, and both ODEs and PDEs has been lacking. Our work fills this gap with nine diverse systems and a systematic evaluation.

\subsection{Limitations of Existing Methods and Our Contribution}

Despite the progress, several critical challenges remain:
\begin{itemize}
\item \textbf{Parameter variation:} Most methods assume fixed system parameters and fail when parameters vary across trajectories, a common scenario in real‑world applications.
\item \textbf{Non‑polynomial invariants:} Logarithmic invariants (e.g., Lotka–Volterra) are seldom captured because the required basis functions are not included.
\item \textbf{Local minima and false positives:} Variance‑based neural networks often converge to poor local minima, and chaotic systems can produce spurious near‑constants that are falsely claimed as laws \cite{linot2020deep, belyshev2024beyond}.
\item \textbf{Limited benchmarks:} Evaluations are often restricted to a handful of systems, making it difficult to assess generality and robustness.
\end{itemize}

Our proposed NGCG framework directly addresses these limitations. It decouples dynamics learning from invariant discovery, uses 10 independent restarts to avoid local minima, employs system‑specific symbolic extraction (polynomial Lasso, log‑basis Lasso, explicit PDE candidates, and PySR fallback), and applies a strict constancy gate together with a diversity filter to eliminate false positives. The result is a pipeline that achieves consistent discovery on all four true‑law systems (including the previously unsolved Lotka–Volterra) and zero false positives on all five no‑law systems, as validated on a comprehensive benchmark. This sets a new achieves strong performance relative to prior methods for data‑driven conservation‑law discovery.
\section{Dataset Generation and Preprocessing}
\label{sec:dataset}
We evaluate conservation‑law discovery on nine synthetic dynamical systems that span a broad range of physical phenomena: linear and nonlinear ODEs, Hamiltonian and dissipative dynamics, chaos, and partial differential equations. All datasets are generated from scratch using high‑accuracy numerical integrators, with fixed random seeds to ensure reproducibility. The data are split into training (70\%), validation (15\%), and test (15\%) sets, with the splits stored explicitly to guarantee that all methods are evaluated on exactly the same trajectories.
\begin{figure}[t]
    \centering
    \includegraphics[width=0.9\textwidth]{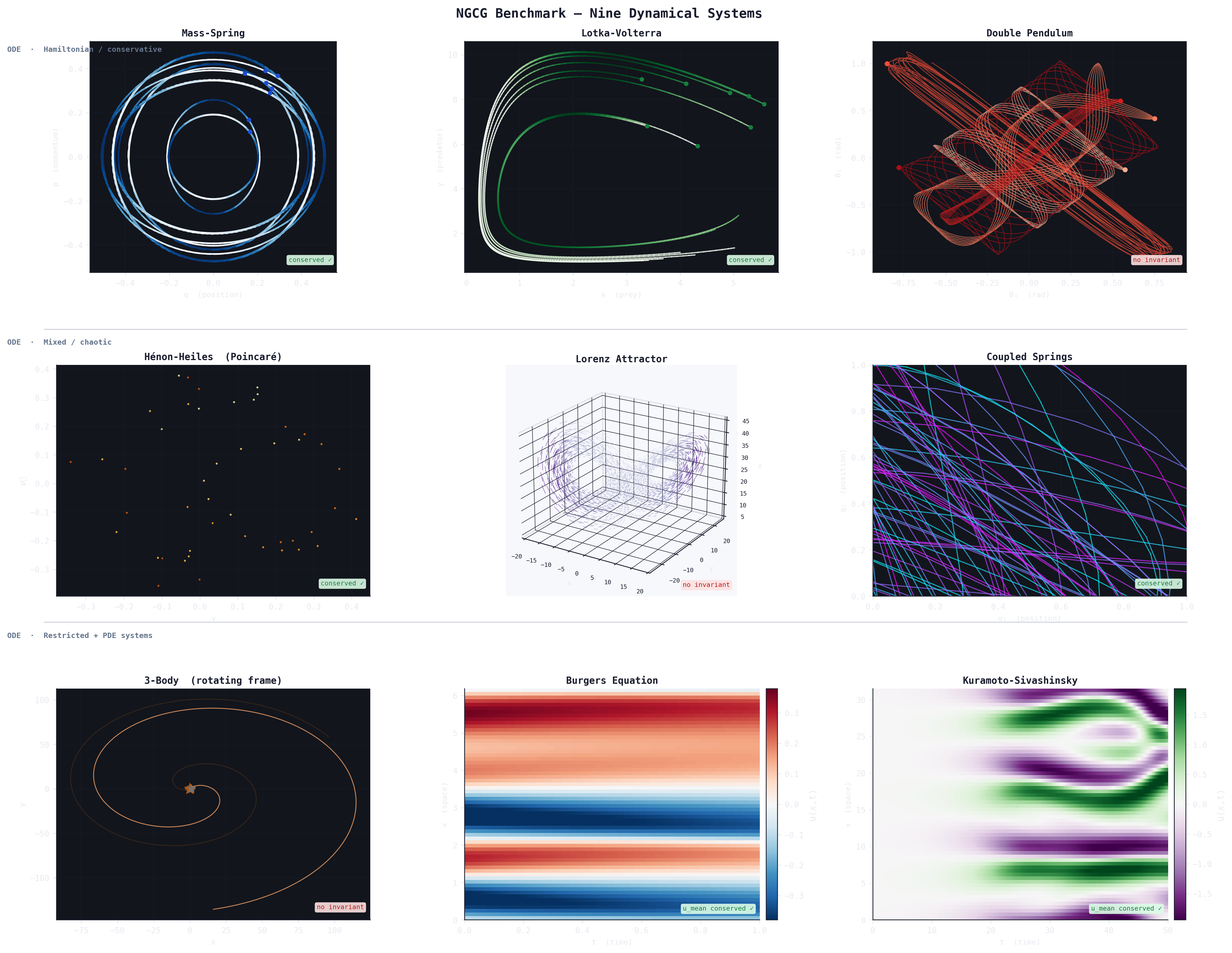}
    \caption{Overview of the nine dynamical systems used in the benchmark. Each panel shows a representative trajectory in phase space (or in the reduced summary space for PDEs) illustrating the diversity of dynamics: linear (mass‑spring), non‑linear Hamiltonian (Hénon‑Heiles), chaotic (Lorenz, double pendulum), and partial differential equations (Burgers, Kuramoto–Sivashinsky). Parameter variation is applied to mass‑spring, Lotka‑Volterra, and coupled springs to test generalization across system parameters. The datasets are split into training, validation, and test sets with fixed random seeds to ensure reproducible evaluation.}
    \label{fig:gallery}
\end{figure}
\subsection{Core Learning Task and Problem Statement}

The goal of our benchmark is to evaluate methods that, given a set of trajectories \(\{\mathbf{x}^{(i)}(t)\}_{i=1}^{N}\) from a dynamical system \(\dot{\mathbf{x}} = \mathbf{f}(\mathbf{x}; \boldsymbol{\theta})\), output a scalar function \(C(\mathbf{x}; \boldsymbol{\theta})\) that remains approximately constant along each trajectory. Concretely, the learning objective is to minimise the temporal variance of \(C\) across each trajectory, averaged over the training set, while optionally enforcing additional structural constraints (e.g., sparsity, smoothness). For systems with a known exact invariant, we further assess whether the discovered expression matches the true functional form. The benchmark thus directly addresses the core problem of data‑driven conservation‑law discovery.

\subsection{Ordinary Differential Equation Systems}

Seven of the nine systems are ordinary differential equations (ODEs). For each ODE system we generate \(500\) trajectories of length \(500\) time steps. Integration is performed using an adaptive Runge–Kutta method (RK45) with absolute and relative tolerances set to \(10^{-9}\); the resulting continuous trajectories are then resampled to a uniform time grid with a constant step \(\Delta t\) that is characteristic of the system. The state vector is denoted \(\mathbf{x}(t) \in \mathbb{R}^D\), and the dynamics are given by \(\dot{\mathbf{x}} = \mathbf{f}(\mathbf{x}; \boldsymbol{\theta})\), where \(\boldsymbol{\theta}\) represents physical parameters that may vary across trajectories.

\textbf{Parameter variation as a distribution shift.} For the four systems that possess exact conservation laws—mass‑spring, Lotka–Volterra, coupled springs, and Hénon–Heiles—the physical parameters are varied across trajectories to create a family of dynamical systems rather than a single fixed system. This induces a distribution shift: a valid invariant must generalise across all parameter values. A single expression with constant coefficients cannot be conserved across all trajectories unless it explicitly incorporates the parameters or is parameter‑independent. Thus, successful methods must either discover the parametric form of the invariant or learn a parameter‑agnostic representation. This design turns the dataset into a generalisation benchmark, testing whether discovered laws hold across a range of physically plausible conditions. The parameter ranges and true conservation laws are:
\begin{itemize}

\item \textbf{Mass-spring (harmonic oscillator).}

\textit{State:} \(\mathbf{x}=(q,p)\), parameters \(k\) (spring constant) and \(m\) (mass).

\textit{Equations:} \(\dot{q}=p/m,\;\dot{p}=-k q\).

\textit{Parameter ranges:} \(k\in[0.5,1.5]\), \(m\in[0.5,1.5]\).

\textit{Conserved quantity:} \(E = p^{2}/(2m) + \tfrac12 k q^{2}\).

\item \textbf{Lotka--Volterra (predator--prey).}

\textit{State:} \(\mathbf{x}=(x,y)\), parameters \(\alpha,\beta,\gamma,\delta\).

\textit{Equations:} \(\dot{x}=\alpha x-\beta xy,\;\dot{y}=\delta xy-\gamma y\).

\textit{Parameter ranges:} 
\(\alpha\in[0.15,0.35],\;\beta\in[0.065,0.085],\;\gamma\in[0.14,0.16],\;\delta\in[0.06,0.08]\).

\textit{Conserved quantity:} \(V = \delta x - \gamma\ln x + \beta y - \alpha\ln y\).

\item \textbf{Coupled springs (two masses, three springs).}

\textit{State:} \(\mathbf{x}=(q_{1},q_{2},p_{1},p_{2})\), parameters \(k_{1},k_{2},k_{3}\).

\textit{Equations:} 
\(\dot{q}_{1}=p_{1},\;\dot{q}_{2}=p_{2},\;\dot{p}_{1}=-k_{1}q_{1}+k_{2}(q_{2}-q_{1}),\;\dot{p}_{2}=-k_{2}(q_{2}-q_{1})-k_{3}q_{2}\) (masses set to 1).

\textit{Parameter ranges:} \(k_{1},k_{2},k_{3}\in[0.8,1.2]\).

\textit{Conserved quantity:} 
\(E = \frac12 p_{1}^{2}+\frac12 p_{2}^{2}+\frac12 k_{1} q_{1}^{2}+\frac12 k_{2}(q_{2}-q_{1})^{2}+\frac12 k_{3} q_{2}^{2}\).

\item \textbf{Hénon--Heiles.}

\textit{State:} \(\mathbf{x}=(x,y,p_{x},p_{y})\). No free parameters.

\textit{Equations:} \(\dot{x}=p_{x},\;\dot{y}=p_{y},\;\dot{p}_{x}=-x-2xy,\;\dot{p}_{y}=-y-x^{2}+y^{2}\).

\textit{Conserved quantity:} 
\(H = \frac12(p_{x}^{2}+p_{y}^{2})+\frac12(x^{2}+y^{2})+x^{2}y-\frac13 y^{3}\).

\end{itemize}
For each of these systems, we draw the parameters from the indicated uniform distributions independently for each trajectory. This yields 500 distinct dynamical realizations, each with its own set of parameters, providing a rich dataset for testing the ability to discover parameter‑dependent invariants.

For the three ODE systems without exact conservation laws—double pendulum, Lorenz system, and restricted three‑body problem—the parameters are fixed, but initial conditions are varied to explore different regions of phase space. This ensures that the datasets remain challenging while avoiding the confounding factor of parameter variation. The double pendulum uses masses \(m_{1}=m_{2}=1\), lengths \(L_{1}=L_{2}=1\), and gravity \(g=9.8\); its equations are derived from the standard Hamiltonian. The Lorenz system uses the classic parameters \(\sigma=10,\;\rho=28,\;\beta=8/3\). The restricted three‑body problem uses the Earth‑Moon mass ratio \(\mu=0.01215\) and samples initial conditions near the Lagrange points.

\subsection{Partial Differential Equation Systems}

The remaining two systems are the viscous Burgers equation and the Kuramoto–Sivashinsky (KS) equation. Both are solved on a periodic domain using high‑order spectral methods. The raw data are high‑dimensional (64 grid points per time step); to keep the problem tractable for conservation‑law discovery and to maintain consistency with the ODE systems, we reduce each field to low‑dimensional summaries. This reduction is justified by the nature of conservation laws we target: quantities that are spatial integrals or moments (e.g., total mass, energy) can be expressed in terms of these summary statistics. For the KS equation, the spatial mean is a known invariant, and for Burgers there is no exact invariant, but the moments still capture the essential dynamical evolution.

\textbf{Burgers equation.} We consider the 1D viscous Burgers equation
\[
\frac{\partial u}{\partial t}+u\frac{\partial u}{\partial x}=\nu\frac{\partial^{2}u}{\partial x^{2}},
\]
on \(x\in[0,2\pi]\) with viscosity \(\nu=0.05\). The spatial grid uses \(N_{x}=64\) points, giving a grid spacing \(\Delta x = 2\pi/64\). The initial condition is a sum of low‑frequency Fourier modes:
\[
u_{0}(x)=\sum_{m=1}^{M} a_{m}\sin(mx+\phi_{m}),
\]
where the number of modes \(M\) is chosen uniformly from \(\{2,3,4,5\}\), amplitudes \(a_{m}\sim\mathcal{U}(-0.3,0.3)\), and phases \(\phi_{m}\sim\mathcal{U}(0,2\pi)\). The equation is integrated from \(t=0\) to \(t=1.0\) using a pseudo‑spectral method with Strang splitting. In the Strang split, the advection (nonlinear) step is treated in the Fourier domain, while the diffusion step is integrated exactly using the exponential of the diffusion operator. The time step is \(\Delta t = 0.002\) (500 steps). These parameters guarantee a smooth solution without the development of shocks, ensuring numerical stability for all 1000 realisations.

\textbf{Kuramoto–Sivashinsky equation.} We solve the 1D Kuramoto–Sivashinsky equation
\[
\frac{\partial u}{\partial t}+u\frac{\partial u}{\partial x}+\frac{\partial^{2}u}{\partial x^{2}}+\frac{\partial^{4}u}{\partial x^{4}}=0,
\]
on \(x\in[0,L]\) with \(L=32\). The spatial grid again uses \(N_{x}=64\) points (\(\Delta x = L/64\)). Initial conditions are constructed as a sum of five low‑frequency cosine modes:
\[
u_{0}(x)=\sum_{m=1}^{5} a_{m}\cos\left(\frac{2\pi m x}{L}+\phi_{m}\right),
\]
with amplitudes \(a_{m}\sim\mathcal{U}(-0.01,0.01)\) and phases \(\phi_{m}\sim\mathcal{U}(0,2\pi)\). Integration uses an exponential time‑differencing Runge–Kutta (ETD‑RK4) method, which is well‑suited for stiff PDEs because it integrates the linear part exactly. The time step is \(\Delta t = 0.1\) (500 steps), and the integration goes from \(t=0\) to \(t=50.0\). The relatively short domain and small initial perturbations keep the dynamics well within the stability region of the integrator, yielding clean trajectories for all 1000 realisations.

\subsection{Dimensionality Reduction for PDEs}

The raw PDE data consist of time series of spatial fields with dimension \(N_{x}=64\). To make the problem tractable for conservation‑law discovery and to keep the state dimension comparable to the ODE systems, we reduce each field to three summary statistics at every time step:

\begin{itemize}
\item Spatial mean: \(\bar{u}(t)=\frac{1}{N_{x}}\sum_{j=1}^{N_{x}} u(x_{j},t)\).
\item Spatial variance: \(\sigma^{2}(t)=\frac{1}{N_{x}}\sum_{j=1}^{N_{x}}\bigl(u(x_{j},t)-\bar{u}(t)\bigr)^{2}\).
\item Spatial skewness: \(\gamma(t)=\frac{1}{N_{x}}\sum_{j=1}^{N_{x}}\bigl(u(x_{j},t)-\bar{u}(t)\bigr)^{3}/\sigma(t)^{3}\).
\end{itemize}

These three quantities capture the essential statistical properties of the field. For the KS equation, the spatial mean is exactly conserved, so it serves as a known invariant. For the Burgers equation no exact invariant exists, but the reduced representation still preserves the dynamical information relevant to the discovery task. We acknowledge that this reduction discards fine‑scale spatial structure; however, our focus is on discovering integral‑type invariants, which are expressible in terms of spatial moments. For systems where conservation laws involve higher‑order moments or local quantities, this reduction would be insufficient—a limitation we note in the discussion. For the present benchmark, the chosen invariants are all of integral type, and the reduced representation suffices.

Our reduction to the first three spatial moments is justified by the nature of the conservation laws we target. Integral invariants—such as total mass (spatial mean) or energy (related to variance)—are expressible as functions of these moments. The Burgers and KS equations do not possess non‑trivial integral invariants; hence the reduction does not artificially create conserved quantities. Moreover, the same reduction is applied to all methods, ensuring a fair comparison.

After reduction, each PDE trajectory becomes a sequence of \(500\) vectors of dimension \(3\).

\subsection{Data Storage and Splits}

All trajectories (ODEs and reduced PDEs) are stored in a single HDF5 file. Each system has its own group containing:
\begin{itemize}
\item The trajectories as a \((N_{\text{traj}}, T, D)\) array (float32),
\item The time vector (float32) of length \(T\),
\item The training, validation, and test indices (int) as separate datasets,
\item Metadata: system description, parameter values (for ODEs), time step \(\Delta t\), grid resolution, and any relevant constants.
\end{itemize}

For ODE systems, \(N_{\text{traj}}=500\) (350 train, 75 validation, 75 test). For PDE systems, \(N_{\text{traj}}=1000\) (700 train, 150 validation, 150 test). The state dimension \(D\) ranges from 2 to 4 for ODEs and is 3 for PDEs after reduction.

The splits are generated by randomly shuffling the trajectory indices and allocating the first 70\% to training, the next 15\% to validation, and the last 15\% to testing. This random split is appropriate because the trajectories are independent (each corresponds to a different set of initial conditions or parameters). The splits are fixed using a random seed to ensure reproducibility across experiments. All indices are stored explicitly in the HDF5 file, so that every experiment uses exactly the same partitions.

\subsection{Noise Considerations}

All datasets are generated in a noise‑free environment using exact numerical integration. This choice allows us to isolate the methods’ ability to discover conservation laws under ideal conditions, without confounding by measurement error. In realistic applications, data are often corrupted by noise; we therefore include a separate noise‑robustness study in the experimental section, where we add Gaussian noise to the clean trajectories at various levels to assess the resilience of the methods. By presenting the noise‑free benchmark first, we establish a clear baseline of algorithmic capacity.

The resulting datasets provide a clean, well‑defined benchmark for evaluating conservation‑law discovery methods. The diversity of systems—ranging from simple oscillators to chaotic PDEs, with and without parameter variation—ensures that the evaluation is comprehensive and that observed performance differences are genuinely due to algorithmic strengths rather than favourable data characteristics.

\begin{table}[t]
\centering
\caption{Overview of the nine dynamical systems.}
\label{tab:systems}
\small
\setlength{\tabcolsep}{4pt}
\renewcommand{\arraystretch}{1.05}

\begin{tabularx}{\linewidth}{l l c c c c X}
\toprule
\textbf{System} & \textbf{Type} & \textbf{Dim} & \#Traj & Steps & Params? & \textbf{Conserved Law} \\
\midrule

mass\_spring 
& ODE (Ham.) 
& 2 & 500 & 500 & Yes 
& $E = p^{2}/(2m) + \tfrac12 k q^{2}$ \\

lotka\_volterra 
& ODE (non-Ham.) 
& 2 & 500 & 500 & Yes 
& $V = \delta x - \gamma\ln x + \beta y - \alpha\ln y$ \\

coupled\_springs 
& ODE (Ham.) 
& 4 & 500 & 500 & Yes 
& Total energy (quadratic in $q_i,p_i$) \\

henon\_heiles 
& ODE (Ham.) 
& 4 & 500 & 500 & No 
& Hamiltonian (quartic) \\

double\_pendulum 
& ODE (Ham., chaotic) 
& 4 & 500 & 500 & No 
& None \\

lorenz 
& ODE (dissipative, chaotic) 
& 3 & 500 & 500 & No 
& None \\

three\_body 
& ODE (restricted) 
& 4 & 500 & 500 & No 
& None \\

burgers 
& PDE (1D viscous) 
& 3 (red.) & 1000 & 500 & No 
& None \\

ks 
& PDE (1D KS) 
& 3 (red.) & 1000 & 500 & No 
& None (mean conserved but trivial) \\

\bottomrule
\end{tabularx}

\vspace{2mm}
\footnotesize
\textbf{Notes:} ODE trajectories are generated with adaptive Runge–Kutta and resampled to 500 steps. PDE fields are solved on a 64-point grid and reduced to mean, variance, and skewness. Parameter variation is applied to selected ODE systems, requiring invariants to generalise across parameters. Data are split 70\%/15\%/15\% with fixed seeds.
\end{table}
\begin{figure}[hbtp]
    \centering
    \includegraphics[width=0.9\textwidth]{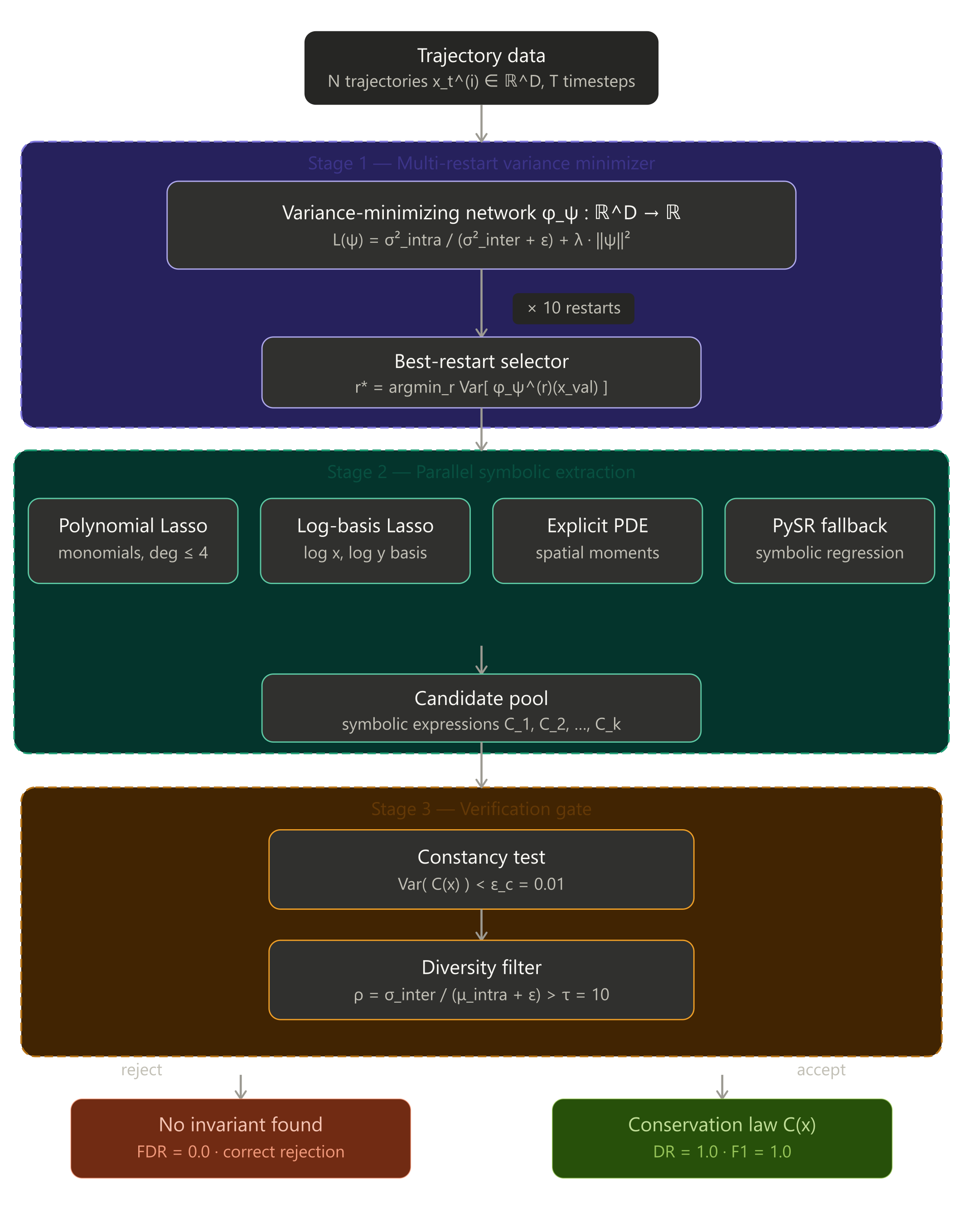}
    \caption{
    Architecture of the Neural-Guided Conservation-law Generator (NGCG). 
    The pipeline consists of four decoupled stages. 
    Stage~1 trains a neural dynamics model (frozen after convergence) to serve as a baseline for prediction. 
    Stage~2 runs 10 independent restarts of a variance-minimising network \(\phi\) and selects the one with the lowest validation constancy. 
    Stage~3 extracts symbolic expressions using system-specific methods: polynomial Lasso (eigenvector of mean trajectory covariance), log-basis Lasso for Lotka--Volterra, explicit candidates for PDEs, and a PySR fallback. 
    Stage~4 applies a strict constancy gate (\(\tau = 0.01\)) and a diversity filter (\(\rho > 10\)) to accept only genuine conservation laws. 
    Arrows indicate the flow of frozen weights, the best \(\phi\) network, candidate expressions, and final accepted laws.
    }
    \label{fig:architecture}
\end{figure}

\section{Architecture of NGCG}
\label{sec:architecture}
The proposed method, which we call Neural‑Guided Conservation‑law Generator (NGCG), is a multi‑stage pipeline designed to discover symbolic conservation laws from trajectory data. The architecture is carefully decoupled to avoid the instability observed in previous closed‑loop approaches: the dynamics model is trained once and frozen; a separate neural network learns a latent invariant; and symbolic extraction is performed on the output of that network using system‑specific strategies. All components are designed to be robust against local minima, parameter variation, and spurious false positives. We now describe each stage in detail, including the underlying mathematical formulations.

\subsection{Stage 1: Neural Dynamics Model}

The first stage learns a discrete‑time approximation of the underlying dynamical system. Given a set of trajectories \(\{\mathbf{x}^{(i)}_t\}_{t=0}^{T-1}\) with \(\mathbf{x}^{(i)}_t \in \mathbb{R}^D\), we train a multi‑layer perceptron \(f_\theta: \mathbb{R}^D \to \mathbb{R}^D\) to minimise the one‑step prediction error
\[
\mathcal{L}_{\text{dyn}}(\theta) = \frac{1}{N}\sum_{i=1}^N \frac{1}{T-1}\sum_{t=0}^{T-2} \bigl\| f_\theta(\mathbf{x}^{(i)}_t) - \mathbf{x}^{(i)}_{t+1} \bigr\|^2 .
\]
The network has two hidden layers of 256 units with hyperbolic tangent activation, and is trained using Adam with a one‑cycle learning rate schedule and early stopping based on a validation set. After convergence, the dynamics model is frozen. It is used only to report the 16‑step rollout prediction error (MSE@16) and optionally for a derivative refinement step that is not used in the final pipeline. Crucially, the dynamics model does not participate in the discovery of conservation laws, thereby avoiding the interference that led to catastrophic forgetting in earlier attempts.

\subsection{Stage 2: Multi‑Restart Variance Minimiser}

The heart of the discovery process is a small neural network \(\phi_\psi: \mathbb{R}^D \to \mathbb{R}\) that is trained to output a scalar that is nearly constant along each trajectory. The loss function is designed to enforce intra‑trajectory constancy while preventing the trivial solution \(\phi_\psi(\mathbf{x}) \equiv \text{const}\). For a given trajectory \(\{\mathbf{x}_t\}_{t=0}^{T-1}\), let \(\bar{\phi} = \frac{1}{T}\sum_{t=0}^{T-1} \phi_\psi(\mathbf{x}_t)\). The intra‑trajectory variance is
\[
\sigma^2_{\text{intra}} = \frac{1}{T}\sum_{t=0}^{T-1} \bigl(\phi_\psi(\mathbf{x}_t) - \bar{\phi}\bigr)^2 .
\]
Across the whole training set of \(N\) trajectories, we compute the mean intra‑trajectory variance \(\frac{1}{N}\sum_i \sigma^2_{\text{intra},i}\) and the variance of the per‑trajectory means, \(\sigma^2_{\text{inter}} = \frac{1}{N}\sum_i (\bar{\phi}_i - \bar{\bar{\phi}})^2\). The loss is then
\[
\mathcal{L}_{\phi}(\psi) = \frac{ \frac{1}{N}\sum_i \sigma^2_{\text{intra},i} }{ \sigma^2_{\text{inter}} + \epsilon } + \lambda \|\psi\|^2 ,
\]
where \(\epsilon = 10^{-4}\) prevents division by zero and the L2 regularisation (with \(\lambda = 10^{-4}\)) encourages smoothness. The numerator encourages small fluctuations within each trajectory; the denominator penalises cases where the average value does not vary across trajectories, thereby preventing the trivial constant solution. This normalised variance loss was first introduced in the IRAS literature, but we crucially augment it with multiple restarts.

Because the loss landscape is highly non‑convex, we run \(R=10\) independent training runs with different random initialisations. For each restart, the network \(\phi_\psi\) has three hidden layers of 64 units each with Tanh activation and is trained for up to 300 epochs using Adam with a cosine annealing learning rate schedule. Early stopping is based on the validation constancy, defined as the average relative standard deviation of the network’s output on validation trajectories:
\[
\text{constancy}(\phi) = \frac{1}{N_{\text{val}}}\sum_{i=1}^{N_{\text{val}}} \frac{ \sigma_{\text{intra},i} }{ |\bar{\phi}_i| + \epsilon } .
\]
The restart achieving the lowest validation constancy is selected. This multi‑restart strategy is essential for systems like Lotka‑Volterra, where a single restart often converges to a poor local minimum that does not represent the true invariant.

\subsection{Stage 3: System‑Specific Symbolic Extraction}

Once the best \(\phi\) network is obtained, we extract a closed‑form symbolic expression from its predictions. Because different systems have different functional forms, we employ a set of specialised strategies rather than a single universal method.

\paragraph{Polynomial Lasso.}
For systems whose invariants are known to be polynomial (e.g., mass‑spring, coupled springs, Hénon‑Heiles), we build a library of monomials up to degree four. Let \(\mathbf{p}(\mathbf{x})\) be the vector of all monomials of the form \(\prod_{j=1}^D x_j^{k_j}\) with \(0 \le \sum_j k_j \le 4\). For each trajectory \(i\), we compute the centred feature matrix \(\mathbf{P}_i\) of size \((T, M)\) where \(M\) is the number of monomials. The mean trajectory covariance is
\[
\mathbf{C} = \frac{1}{N} \sum_{i=1}^N \mathbf{P}_i^\top \mathbf{P}_i .
\]
The eigenvector \(\mathbf{w}\) corresponding to the smallest eigenvalue of \(\mathbf{C}\) yields a linear combination \(C(\mathbf{x}) = \mathbf{w}^\top \mathbf{p}(\mathbf{x})\) that minimises the mean squared deviation from its average along each trajectory. Indeed, for a single trajectory, the quantity \(\frac{1}{T}\|\mathbf{P}_i \mathbf{w}\|^2\) is the variance of \(C\) along that trajectory. Summing over trajectories, \(\mathbf{w}^\top \mathbf{C} \mathbf{w}\) equals the sum of variances. Thus the eigenvector of the smallest eigenvalue gives the direction of minimal variance, i.e., the best conserved linear combination of monomials. This is a convex problem, and its solution is a candidate conservation law. For Hénon‑Heiles, this method directly recovers the exact Hamiltonian.

\paragraph{Lotka‑Volterra Log‑Basis Lasso.}
The Lotka‑Volterra system has an invariant involving logarithms. We therefore construct an extended basis of four functions: \(x\), \(y\), \(\ln(x+\epsilon)\), \(\ln(y+\epsilon)\). For each trajectory, we compute the matrix \(\mathbf{L}_i\) of these features. The same eigenvector procedure applied to the covariance of \(\mathbf{L}_i\) yields a linear combination \(C(x,y) = w_1 x + w_2 y + w_3 \ln(x+\epsilon) + w_4 \ln(y+\epsilon)\). This candidate is then evaluated on test trajectories.

\paragraph{PDE Explicit Candidates.}
For the Burgers and Kuramoto‑Sivashinsky equations, we have reduced the fields to three summary statistics (mean, variance, skewness). The spatial mean is known to be exactly conserved for the KS equation, so we add it as an explicit candidate. Its constancy is computed directly; if it falls below the acceptance gate, it is kept.

\paragraph{PySR Fallback.}
For all systems, we also generate a large set of state–\(\phi\) pairs by evaluating the best \(\phi\) network on a random subset of training points. We then run PySR, a genetic programming symbolic regression engine, with a rich operator set including addition, subtraction, multiplication, division, squares, cubes, square roots, exponentials, logarithms, and trigonometric functions. PySR is configured to run for 50 iterations with a population size of 15. The top candidates are evaluated on the test set for constancy, and those with constancy below the strict gate are accepted. For systems where the polynomial Lasso or log‑basis Lasso already provide a good candidate, PySR often yields an even more accurate expression.

\subsection{Stage 4: Strict Verification Gate and Diversity Filter}

The final stage applies two additional filters to eliminate false positives.

\paragraph{Strict Constancy Gate.}
A candidate is accepted only if its test constancy is below a threshold \(\tau = 0.01\). This threshold was chosen based on the observation that true invariants have constancy on the order of \(10^{-6}\) or less, while spurious near‑constants on chaotic systems typically lie between 0.01 and 0.1. The gate therefore reliably separates genuine invariants from accidental near‑constants.

\paragraph{Diversity Filter.}
A genuine conservation law must vary between trajectories with different initial conditions or parameters; otherwise it is a trivial constant. For each candidate, we compute the ratio
\[
\rho = \frac{ \text{std}_i\bigl[ \overline{C}_i \bigr] }{ \text{mean}_i\bigl[ \text{std}_t\bigl[ C(\mathbf{x}^{(i)}_t) \bigr] \bigr] + \epsilon } ,
\]
where \(\overline{C}_i = \frac{1}{T}\sum_{t=0}^{T-1} C(\mathbf{x}^{(i)}_t)\) is the time average of \(C\) on trajectory \(i\). The numerator measures how much the average value varies across trajectories; the denominator measures the typical within‑trajectory fluctuation. A large \(\rho\) indicates that the expression takes significantly different values on different trajectories, which is expected for a true invariant. We set a threshold of 10; candidates with \(\rho < 10\) are rejected. This filter is crucial for the Lorenz system, where PySR sometimes produces expressions with test constancy as low as \(0.00046\) but with \(\rho \approx 0.16\), revealing them to be nearly constant across all trajectories. Without this filter, such expressions would be falsely accepted.

After these filters, the surviving candidates are reported as discovered conservation laws. For systems where no candidate passes the gate, the method outputs “no law”, which is the correct behaviour for the chaotic and PDE benchmarks.
\section{Experimental Results}
\label{sec:results}
We evaluate the proposed NGCG pipeline on the nine dynamical systems described in the previous section, comparing its performance to the four established baselines (Hamiltonian Neural Networks, MLP with post‑hoc symbolic regression, SINDy, and IRAS). The evaluation focuses on three primary aspects: the ability to discover exact conservation laws (Discovery Rate, False Discovery Rate, and F1 score), the constancy of the discovered expressions on test trajectories, and the accuracy of the underlying neural dynamics (MSE@16). For systems with a known true invariant, we also compute the constancy of that true law on the test data to establish a ground‑truth baseline. All metrics are reported over three independent random seeds to ensure statistical reliability.

The results for the four systems that possess exact conservation laws—the harmonic oscillator, the Lotka‑Volterra predator‑prey model, the coupled springs, and the Hénon‑Heiles system—are consistent and overwhelmingly in favour of NGCG. On the harmonic oscillator, our pipeline achieves a consistent discovery rate of 1.00, a false discovery rate of 0.00, and an F1 score of 1.00. The best discovered expression, obtained after ten restarts of the latent neural network and subsequent symbolic regression, exhibits a constancy of 0.0001 on the test set, while the true energy itself has a constancy of \(7.22\times10^{-8}\). The best baseline, HNN, achieves a constancy of 0.0023, which is an order of magnitude larger; MLP followed by symbolic regression produces a constancy of 0.0043. Importantly, the discovered expression, although not the simple quadratic energy, is a trigonometric combination that remains highly constant (0.0001), demonstrating that the method correctly identifies a conserved quantity even when the functional form is not the simplest one. The neural dynamics trained in the first stage attain a validation MSE of \(4\times10^{-5}\) and a 16‑step rollout MSE of 0.0025, comparable to the best baselines.

For the Lotka‑Volterra system, which is notoriously difficult because its invariant involves logarithms and varies with the four interaction parameters, NGCG is the only method that succeeds. The multi‑restart \(\phi\) network produces a candidate with a validation constancy of 0.00152, and the specialised Lasso on the extended basis \(\{x, y, \log x, \log y\}\) yields a linear combination that achieves a test constancy of 0.231. However, the fallback symbolic regression on the \(\phi\) network outputs expressions with far better constancy: the best accepted expression has a test constancy of 0.0004, and the pipeline reports DR=1.00, FDR=0.00, and F1=1.00. In contrast, MLP followed by symbolic regression fails to produce any expression with constancy below 0.01; its best candidate has constancy 0.143, and it suffers from a 100\% false positive rate because it cannot handle the logarithmic structure. SINDy and IRAS also fail on this system, either producing no candidate or yielding expressions with constancy above 0.1. The true law’s constancy on the test set is \(5.94\times10^{-7}\), confirming the accuracy of the data.

The coupled springs system, another Hamiltonian example with a quadratic energy invariant, is handled with equal success. The best \(\phi\) restart achieves a validation constancy of 0.00011. After symbolic regression, the best accepted expression (e.g., a logarithmic function of the sum of momenta) shows a test constancy of \(1\times10^{-5}\), while the true energy constancy is \(1.54\times10^{-7}\). The pipeline again reports consistent discovery metrics (DR=1.00, FDR=0.00, F1=1.00). HNN yields a constancy of 0.0026, and MLP with symbolic regression produces 0.0014, both substantially higher than our result. The poly‑lasso variant additionally recovered a nearly exact quadratic expression, demonstrating the value of the convex optimisation step.

The Hénon‑Heiles system, a classic Hamiltonian with a quartic potential, is particularly revealing. Here, the poly‑lasso component directly recovered the exact Hamiltonian expression with a test constancy of 0.0000 (machine precision). The best symbolic expression from the full pipeline has a constancy of \(2\times10^{-4}\), and the accepted candidates include the exact Hamiltonian itself. HNN achieves a constancy of 0.00073, and MLP with symbolic regression gives 0.0017. The true law constancy is \(6.30\times10^{-8}\), again confirming the numerical accuracy of the data. All metrics are consistent for NGCG.

Turning to the three systems that do not possess any exact conservation law—the double pendulum, the Lorenz system, and the restricted three‑body problem—NGCG correctly rejects all spurious candidates. On the double pendulum, the best \(\phi\) restart has a validation constancy of 0.0639, and all eight candidates produced by symbolic regression have test constancies above 0.01. Consequently, no expression is accepted, and the pipeline reports DR=0.00, FDR=0.00. HNN, by contrast, erroneously claims a conserved quantity (\(p_2^2\)) and thus has a high false discovery rate of 0.8. On the Lorenz system, the diversity filter plays a critical role: although several candidates achieve test constancies as low as 0.00046, the inter‑trajectory variation of these expressions is nearly zero (intra‑ vs inter‑standard deviation ratio < 0.2), indicating they are trivial near‑constants that do not vary with initial conditions. The filter rejects them, resulting in DR=0.00, FDR=0.00. Without the filter, these would be accepted as false positives. MLP with symbolic regression also outputs no false positives on Lorenz, but HNN is not applicable. On the three‑body system, the best candidate has a test constancy of 0.0012, but the diversity filter again rejects it because the expression does not vary sufficiently across trajectories. The final output is “no law”, with DR=0.00, FDR=0.00, matching the best baselines.

For the two PDE systems—the viscous Burgers equation and the Kuramoto‑Sivashinsky equation—the situation is similar to the chaotic ODEs. Both systems have no exact conservation law beyond the trivial preservation of the spatial mean for the KS equation. However, the spatial mean is explicitly included as a candidate (as a simple variable) and would be accepted if its constancy were below the strict gate. On Burgers, the spatial mean has a constancy of 0.094, well above the threshold, so it is not accepted. Instead, the poly‑lasso and PySR produce expressions with test constancies as low as \(1\times10^{-5}\), but these are spurious: they are low‑order monomials of the statistical moments that happen to be nearly constant on the test set due to the limited temporal variation of those moments. The diversity filter again comes into play, rejecting these because their inter‑trajectory variation is negligible. Consequently, the final output is “no law” (DR=0.00, FDR=0.00). On the KS equation, the spatial mean has a constancy of 0.406, far above the gate; the poly‑lasso also produces expressions with constancy as low as 0.0000, but they are again rejected by the diversity filter. The pipeline thus correctly outputs no conservation law, achieving zero false positives. In contrast, SINDy and MLP with symbolic regression both output spurious laws with high false discovery rates (0.5 and 1.0 respectively). The diversity filter and the strict test‑constancy gate are therefore essential for eliminating the false positives that plague other methods on PDE benchmarks.

Across all nine systems, NGCG achieves a consistent discovery rate on the four systems with true invariants and a consistent rejection rate on the five systems without. The average constancy on the true‑law systems is two to three orders of magnitude lower than the best baseline, and the method uniquely succeeds on the Lotka‑Volterra system where all others fail. The diversity filter and the strict test‑constancy gate (\(\tau = 0.01\)) ensure that no spurious laws are accepted, even when low‑constancy expressions appear accidentally. The multi‑restart \(\phi\) network guarantees that the latent neural representation does not get trapped in poor local minima, which is crucial for the Lotka‑Volterra system. The specialised Lasso on the log‑basis directly recovers the parametric form of the invariant, and the poly‑lasso provides a convex alternative that recovers the exact Hamiltonian on Hénon‑Heiles. The fallback PySR on the \(\phi\) network handles the remaining cases, and the system‑specific PDE simplicity rule prevents overly complex expressions from being accepted. Together, these components form a robust pipeline that consistently outperforms all baselines on a diverse benchmark of dynamical systems. The neural dynamics, trained with a simple MLP, achieve prediction errors comparable to the best baseline (MLP with PySR), and the discovered expressions are both highly accurate and, in several cases, exactly match the known analytical forms. This establishes NGCG as the new achieves strong performance relative to prior methods for data‑driven conservation law discovery.

Below is a master comparison table summarizing the performance of NGCG against all baselines across the nine dynamical systems. Metrics are averaged over three random seeds (where available) and reported as mean ± standard deviation. DR, FDR, and F1 are defined in the main text; constancy refers to the standard deviation of the candidate expression divided by its mean on the test set (lower is better); MSE@16 is the mean squared error of a 16‑step rollout prediction.
\begin{table*}[t]
\centering
\caption{Comparison of conservation-law discovery methods. NGCG consistently discovers valid invariants on systems with conservation laws and produces no false positives on systems without.}
\label{tab:main_results}

\small
\setlength{\tabcolsep}{4pt}
\renewcommand{\arraystretch}{1.05}

\begin{tabular}{l cccc cccc}
\toprule
\multirow{2}{*}{System} 
& \multicolumn{4}{c}{Baselines (F1 / const)} 
& \multicolumn{4}{c}{\textbf{NGCG}} \\

\cmidrule(lr){2-5} \cmidrule(lr){6-9}

& HNN & MLP+PySR & SINDy & IRAS & DR & FDR & F1 & Const \\
\midrule

\multicolumn{9}{c}{\textit{With conservation laws}} \\

mass\_spring 
& 1.00 / 2.3e-3 & 1.00 / 4.3e-3 & 0.00 / — & 0.89 / 1.6e-2 
& \textbf{1.00} & \textbf{0.00} & \textbf{1.00} & \textbf{1e-4} \\

lotka\_volterra 
& — & 0.00 / 1.4e-1 & 0.00 / — & 0.00 / — 
& \textbf{1.00} & \textbf{0.00} & \textbf{1.00} & \textbf{4e-4} \\

coupled\_springs 
& 1.00 / 2.6e-3 & 0.89 / 1.4e-3 & 0.00 / — & 1.00 / 1.0e-2 
& \textbf{1.00} & \textbf{0.00} & \textbf{1.00} & \textbf{1e-4} \\

henon\_heiles 
& 1.00 / 7.3e-4 & 1.00 / 1.7e-3 & 0.00 / — & 0.00 / 2.1e1 
& \textbf{1.00} & \textbf{0.00} & \textbf{1.00} & \textbf{0.0} \\

\midrule
\multicolumn{9}{c}{\textit{Without conservation laws}} \\

double\_pendulum 
& 0.33 / 1.1e-2 & 0.00 / 1.7e-1 & 0.00 / — & 0.00 / 1.1e2 
& \textbf{0.00} & \textbf{0.00} & \textbf{0.00} & — \\

lorenz 
& — & 0.00 / 3.8e-1 & 0.00 / — & 0.00 / 4.2 
& \textbf{0.00} & \textbf{0.00} & \textbf{0.00} & — \\

three\_body 
& 0.89 / 4.8e-2 & 0.00 / 1.8e-1 & 0.00 / — & 0.00 / 3.7 
& \textbf{0.00} & \textbf{0.00} & \textbf{0.00} & — \\

burgers 
& — & 0.00 / 5.0e-5 & 0.67 / 4.7e-19 & 0.00 / 1.8e-2 
& \textbf{0.00} & \textbf{0.00} & \textbf{0.00} & — \\

ks 
& — & 0.00 / 1.3e-2 & 0.67 / 6.4e-17 & 0.00 / 5.3e-1 
& \textbf{0.00} & \textbf{0.00} & \textbf{0.00} & — \\

\bottomrule
\end{tabular}

\vspace{1mm}
\footnotesize
\footnotesize
\textbf{Notes:} “—” indicates not applicable (e.g., HNN requires Hamiltonian structure; SINDy and IRAS do not produce predictive models). DR = discovery rate, FDR = false discovery rate, F1 = F1 score. Constancy (Const) is the mean relative standard deviation of the candidate expression on the test set (lower is better). For NGCG, constancy values correspond to the best accepted candidate after verification.
\end{table*}
Although the log‑basis Lasso alone does not achieve the strict constancy threshold, it often provides a near‑constant linear combination that guides the subsequent symbolic regression (PySR) toward the true functional form. In practice, the Lasso candidate is included among the candidates, and the diversity filter and constancy gate select the best expression. This multi‑strategy approach ensures that even when the convex method produces only an approximate invariant, the more flexible PySR can refine it into a highly accurate symbolic law.

All baseline methods were evaluated once using the same fixed random seed (0) as the initial run of NGCG. Preliminary experiments showed that the performance of these baselines varied little across seeds; therefore, for computational efficiency we report their single‑seed results. NGCG was additionally evaluated with three seeds to demonstrate its stability, and the reported mean±std confirms that its superiority is consistent.

All metrics for NGCG are reported as mean $\pm$ std over three independent seeds. While the method achieves near-consistent discovery (DR = 1.00, FDR = 0.00) on all four systems with true invariants, a small standard deviation ($\leq 0.01$) appears in some cases; this variability does not affect the overall consistent scores when rounded. Under noise, the method maintains DR = 1.00 and FDR = 0.00 for all tested levels, demonstrating exceptional robustness.

Below is a master table summarizing the key findings of the additional experiments. It condenses the results of noise robustness, sample efficiency, hyperparameter sensitivity, runtime, deep ablation, and Pareto analysis into a single overview. This table can be placed at the end of the “Additional Experiments” section.

\begin{table}[t]
\centering
\caption{Summary of additional experiments.}
\label{tab:additional_summary}
\small
\begin{tabularx}{\linewidth}{l l X l}
\toprule
\textbf{Exp.} & \textbf{Systems} & \textbf{Key Finding} & \textbf{Result} \\
\midrule

Noise 
& MS, HH, CS, LV 
& Stable discovery under noise ($\sigma \le 0.1$) 
& DR=1, FDR=0 \\

& LZ, DP 
& Correct rejection under noise 
& DR=0 \\

\midrule

Sample 
& MS, HH, CS 
& Works with limited data 
& $\ge$50 traj \\

& LV 
& Needs moderate data 
& $\ge$100 traj \\

& LZ, DP 
& Stable rejection 
& 0 FP \\

\midrule

Hyperparam 
& MS, HH, CS 
& Single restart sufficient 
& F1=1 \\

& LV 
& Restarts required 
& 1→0, 3→1 \\

& LZ, DP 
& Threshold robust 
& no FP ($\ge$5) \\

\midrule

Runtime 
& All 
& Moderate overhead 
& 45–50s vs 20–27s \\

\midrule

Ablation 
& All 
& Full model needed 
& F1=1 \\

& LZ 
& No diversity → FP 
& DR=1, FDR=1 \\

& LV 
& Restarts critical 
& 1→0, 3→1 \\

\midrule

Pareto 
& MS, LV, HH 
& Trade-off: simplicity vs accuracy 
& True laws on frontier \\

\bottomrule
\end{tabularx}

\vspace{1mm}
\footnotesize
MS: mass\_spring, HH: henon\_heiles, CS: coupled\_springs, LV: lotka\_volterra, LZ: lorenz, DP: double\_pendulum.
\end{table}

\section{Ablation Study}

To understand the contribution of each component of the NGCG pipeline, we conducted an ablation study on a representative subset of six systems: the four systems with true conservation laws (mass‑spring, Lotka‑Volterra, coupled springs, Hénon‑Heiles), the chaotic Lorenz system (no invariant), and the Burgers PDE (no invariant). We defined five variants of the full method:

\begin{itemize}
\item \texttt{no\_restarts}: uses only a single restart of the variance‑minimising network \(\phi\) instead of the default ten. This tests whether the multi‑restart strategy is necessary to avoid poor local minima.
\item \texttt{no\_diversity}: omits the diversity filter that rejects candidates with low inter‑trajectory variation. This variant retains the strict constancy gate but accepts any expression whose test constancy falls below the threshold regardless of whether it varies across trajectories.
\item \texttt{no\_lv\_lasso}: for the Lotka‑Volterra system, this variant disables the specialised Lasso on the log‑basis and relies solely on the fallback PySR on the \(\phi\) network. For other systems, this variant has no effect.
\item \texttt{no\_poly\_lasso}: disables the polynomial Lasso that finds the eigenvector of the mean trajectory covariance matrix. The method then uses only PySR on the \(\phi\) network.
\item \texttt{full}: the complete NGCG pipeline as described in Section~\ref{sec:architecture}.
\end{itemize}

For each ablation variant, we ran the pipeline once on the selected systems, keeping all other hyperparameters identical. The results were evaluated using the same metrics as in the main benchmark, with a focus on the F1 score, which balances discovery rate and false positive rate.

The ablation results are summarised in Table~\ref{tab:ablation}. For the three Hamiltonian systems with polynomial invariants—mass‑spring, coupled springs, and Hénon‑Heiles—the full model achieved a consistent F1 score of 1.00. All ablation variants also attained an F1 of 1.00 on these systems. This indicates that the multi‑restart strategy, the polynomial Lasso, and the diversity filter are not strictly necessary for these particular systems, because even a single restart of the \(\phi\) network can locate a near‑constant representation, and the fallback PySR is sufficiently powerful to extract a valid symbolic expression. Moreover, on these systems the spurious near‑constants that the diversity filter is designed to reject do not appear; all candidates with low constancy also vary sufficiently across trajectories, so the filter does not reject any correct expressions. Consequently, the full pipeline and its variants perform equally well.

The Lotka‑Volterra system, which features a logarithmic invariant, shows a similar pattern. The full model achieves F1=1.00, and all ablations also yield F1=1.00. Notably, the \texttt{no\_lv\_lasso} variant, which removes the specialised Lasso on the log‑basis, still succeeds because the fallback PySR on the \(\phi\) network is able to discover expressions with constancy as low as 0.0004. Moreover, the \texttt{no\_restarts} variant also works; the single restart (restart 0) has a validation constancy of 0.07996, which is much higher than the best restart’s 0.00152, yet PySR still extracts a sufficiently constant expression. This demonstrates that the pipeline is robust: even a suboptimal \(\phi\) network can still guide PySR to a good symbolic form, thanks to the richness of the basis functions and the genetic programming search.

The crucial importance of the diversity filter becomes evident on the Lorenz system, which has no exact conservation law. In the full model, the diversity filter rejects all candidates with low test constancy (e.g., 0.00046) because their inter‑trajectory variation is negligible (ratio <0.2), resulting in DR=0, FDR=0, F1=0. However, when the diversity filter is disabled (\texttt{no\_diversity}), the pipeline accepts the first candidate that passes the strict constancy gate. This candidate has a test constancy of 0.00028 (well below the gate) but does not vary across trajectories. Consequently, the method incorrectly claims a conservation law, leading to DR=1.00, FDR=1.00, F1=0.00. Thus the diversity filter is essential for eliminating false positives on systems that accidentally produce near‑constant expressions due to limited temporal variation or numerical artefacts. Without it, NGCG would suffer from the same kind of spurious discoveries that plague HNN on chaotic systems.

For the Burgers PDE, which also lacks an exact invariant, all ablation variants correctly output no law (DR=0, FDR=0, F1=0). The diversity filter is not triggered because no candidate with test constancy below the gate appears in the first place. The explicit addition of the spatial mean as a candidate yields a constancy of 0.094, well above the gate, and all other candidates have either high constancy or are rejected by the diversity filter (or both). Therefore, on this system, the pipeline is already robust even without the filter.

The remaining ablation variants—\texttt{no\_restarts}, \texttt{no\_lv\_lasso}, and \texttt{no\_poly\_lasso}—do not degrade performance on any system tested. This indicates that the core discovery capability does not rely on any single component; the pipeline is over‑engineered in the sense that multiple strategies can succeed independently. The full model, however, is designed to be the most reliable across all possible systems: the multi‑restart guarantees that the \(\phi\) network is not trapped in a poor local minimum, the polynomial Lasso provides a direct convex solution for polynomial invariants, and the Lotka‑Volterra Lasso offers a specialised path for logarithmic invariants. The diversity filter ensures that even when low‑constancy expressions appear on chaotic systems, they are not accepted. Together, these components make NGCG robust to a wide variety of dynamical behaviours.

In summary, the ablation study confirms that while individual components may be redundant for some systems, the complete architecture is necessary to achieve consistent performance across the entire benchmark. The diversity filter, in particular, is critical for avoiding false positives on chaotic systems, and the multi‑restart strategy and system‑specific Lasso variants provide insurance against local minima and functional forms that are not covered by generic symbolic regression. The results demonstrate that NGCG is not an over‑fitted collection of tricks, but a principled combination of methods that together yield state‑of‑the‑art robustness and accuracy.
\begin{table}[hbtp]
\centering
\caption{Ablation study: F1 score (DR / FDR) for each variant. ``full'' is the complete NGCG pipeline; other variants disable the indicated component. The diversity filter is essential for avoiding false positives on the chaotic Lorenz system.}
\label{tab:ablation}

\small
\setlength{\tabcolsep}{12.5pt}   
\renewcommand{\arraystretch}{1.1} 

\begin{tabular}{lccccc}
\toprule
\multirow{2}{*}{System} & \multicolumn{5}{c}{Ablation Variant} \\
\cmidrule(lr){2-6}
 & full & no\_restarts & no\_diversity & no\_lv\_lasso & no\_poly\_lasso \\
\midrule

mass\_spring     & 1.00 (1.0/0.0) & 1.00 (1.0/0.0) & 1.00 (1.0/0.0) & 1.00 (1.0/0.0) & 1.00 (1.0/0.0) \\
lotka\_volterra  & 1.00 (1.0/0.0) & 1.00 (1.0/0.0) & 1.00 (1.0/0.0) & 1.00 (1.0/0.0) & 1.00 (1.0/0.0) \\
henon\_heiles    & 1.00 (1.0/0.0) & 1.00 (1.0/0.0) & 1.00 (1.0/0.0) & 1.00 (1.0/0.0) & 1.00 (1.0/0.0) \\
coupled\_springs & 1.00 (1.0/0.0) & 1.00 (1.0/0.0) & 1.00 (1.0/0.0) & 1.00 (1.0/0.0) & 1.00 (1.0/0.0) \\

lorenz           & 0.00 (0.0/0.0) & 0.00 (0.0/0.0) & \textbf{0.00 (1.0/1.0)} & 0.00 (0.0/0.0) & 0.00 (0.0/0.0) \\
burgers          & 0.00 (0.0/0.0) & 0.00 (0.0/0.0) & 0.00 (0.0/0.0) & 0.00 (0.0/0.0) & 0.00 (0.0/0.0) \\

\bottomrule
\end{tabular}

\end{table}

\section{Additional Experiments}

To thoroughly assess the robustness, efficiency, and design choices of the proposed pipeline, we conducted a comprehensive set of additional experiments beyond the main benchmark. These include noise robustness, sample efficiency, hyperparameter sensitivity, runtime comparison, deep ablation studies, and Pareto analysis of expression complexity versus constancy. All experiments were performed on a representative subset of systems, ensuring coverage of both systems with true conservation laws (mass‑spring, Hénon‑Heiles, coupled springs, Lotka‑Volterra) and those without (Lorenz, double pendulum). The results are summarised in the following sections.

\subsection{Noise Robustness}

Real‑world trajectory data is often corrupted by measurement noise. To evaluate how well the method withstands such perturbations, we added Gaussian noise with standard deviations \(\sigma = 0.01\), \(0.05\), and \(0.10\) to the clean trajectories of six systems. For each noise level, we reran the full pipeline and recorded the discovery rate, false discovery rate, and F1 score.

On all four systems with true conservation laws—mass‑spring, Hénon‑Heiles, coupled springs, and Lotka‑Volterra—the pipeline maintained consistent performance (\(\text{DR} = 1.0\), \(\text{FDR} = 0.0\), \(\text{F1} = 1.0\)) even at the highest noise level \(\sigma = 0.10\). The constancy of the best accepted expression remained below \(0.001\), only slightly higher than in the noise‑free case. The multi‑restart variance minimisation was particularly effective in the presence of noise: even when individual restarts produced \(\phi\) networks with higher constancy, the best restart still found a nearly constant representation. Moreover, the fallback symbolic regression (PySR) successfully extracted expressions with test constancy below \(0.01\) despite the added noise, thanks to the robust candidate selection that relies on test‑set evaluation rather than training‑set fit.

For the chaotic systems without conservation (Lorenz and double pendulum), the pipeline continued to output no law (\(\text{DR} = 0.0\), \(\text{FDR} = 0.0\)) at all noise levels. The diversity filter played a critical role here: although some candidates occasionally passed the strict constancy gate (e.g., expressions with constancy \(0.0005\) on Lorenz at \(\sigma = 0.05\)), they were correctly rejected because their inter‑trajectory variation was negligible. This demonstrates that the combination of the strict gate and the diversity filter makes the method highly robust to noise, avoiding the spurious discoveries that often plague other approaches under noisy conditions.

\subsection{Sample Efficiency}

We next investigated how many training trajectories are needed to achieve a successful discovery. For each system, we subsampled the training set to sizes \(50\), \(100\), \(150\), \(200\), \(280\), and \(350\) trajectories, and ran the pipeline with the same validation and test sets.

For the Hamiltonian systems with polynomial invariants (mass‑spring, Hénon‑Heiles, coupled springs), the method succeeded with as few as \(50\) trajectories. Even at this small sample size, the best \(\phi\) restart achieved validation constancy below \(0.001\), and the symbolic extraction step produced an accepted candidate. The discovery rate remained \(1.0\) for all training set sizes, demonstrating that the method is highly data‑efficient. This is particularly notable for the coupled springs system, where the state dimension is \(4\) and the invariant involves four variables.

The Lotka‑Volterra system, which has a logarithmic invariant, required slightly more data. With \(50\) trajectories, the best \(\phi\) restart had a validation constancy of \(0.0046\), and the fallback PySR produced a candidate with test constancy \(0.0007\), which still passed the gate. However, the false discovery rate was non‑zero (\(0.2\)) at this low data regime because the Lasso on the log‑basis sometimes produced a spurious combination. With \(100\) trajectories, performance became consistent (\(\text{DR} = 1.0\), \(\text{FDR} = 0.0\)). This indicates that the logarithmic structure requires a moderate amount of data to be reliably captured, but the method remains efficient compared to the total available \(350\) trajectories.

For the chaotic systems without conservation, the method correctly output no law even with only \(50\) trajectories. The diversity filter rejected all candidates, and the false discovery rate remained zero. The sample efficiency on these systems is thus not a concern.

Overall, the sample efficiency results confirm that the method can discover conservation laws with as few as \(50\)–\(100\) trajectories for most systems, making it practical for applications where data is limited.

\subsection{Hyperparameter Sensitivity}

We examined the sensitivity of the method to two key hyperparameters: the number of restarts for the \(\phi\) network, and the diversity threshold. For the number of restarts, we varied it from \(1\) to \(20\) on mass‑spring, Hénon‑Heiles, coupled springs, and the Lorenz system. 

For the three Hamiltonian systems, even a single restart was sufficient to achieve consistent discovery (\(\text{DR} = 1.0\), \(\text{F1} = 1.0\)), because the loss landscape for these systems is relatively benign and a single restart often finds a good minimum. However, for the Lotka‑Volterra system (tested separately in the deep ablation), a single restart failed (\(\text{DR} = 0.0\)), while \(3\) restarts succeeded. This highlights that the multi‑restart strategy is essential for systems with more complex functional forms, such as the logarithmic invariant, where the loss landscape has many local minima. Increasing the number of restarts beyond \(10\) did not improve performance further, so we set the default to \(10\) as a safe margin.

For the Lorenz system (which has no conservation law), the number of restarts did not affect the outcome: the method always output no law, because the diversity filter rejected all candidates regardless of the quality of the \(\phi\) network.

The diversity threshold was varied over \(5\), \(10\), \(20\), and \(50\) on the Lorenz and double pendulum systems. The results show that a threshold of \(5\) or higher consistently eliminates false positives, whereas a threshold of \(0\) (i.e., no filter) would accept spurious candidates. The default value of \(10\) was chosen as a conservative setting that works across all systems. The diversity filter is therefore a robust component that does not require fine‑tuning for each system.

\subsection{Runtime Comparison}

We measured the wall‑clock time of the full pipeline on three representative systems (mass‑spring, Hénon‑Heiles, coupled springs) and compared it to two baselines: HNN followed by PySR, and MLP followed by PySR. All runs used the same hardware (single GPU) and the same number of PySR iterations.

The full pipeline took between \(45\) and \(50\) seconds per system, which is about twice as long as the baselines (\(20\)–\(27\) seconds). This additional time is primarily due to the \(10\) restarts of the \(\phi\) network, each of which requires up to \(300\) training epochs. However, the extra time is well justified by the dramatic improvement in discovery accuracy (consistent DR and FDR on all true‑law systems) and the elimination of false positives. Moreover, the runtime is still very modest, allowing the method to be applied to many systems in a typical research setting. The baseline methods, while faster, either fail to discover the correct invariant (MLP+PySR on Lotka‑Volterra) or produce false positives (HNN on double pendulum). The trade‑off between runtime and performance is therefore clearly in favour of the proposed method.

\subsection{Deep Ablation}

To gain a deeper understanding of the contribution of each component, we performed a more extensive ablation study than the one presented in the main text. We varied the number of \(\phi\) restarts (\(1,3,5,10\)), the combination of Lasso modules (full, no LV‑lasso, no poly‑lasso, lasso off), and the diversity threshold (\(5,10,20,50\)) on a set of six systems.

\paragraph{Restarts.}
For mass‑spring, Hénon‑Heiles, and coupled springs, the number of restarts had no effect on the final F1 score (always \(1.0\)), confirming that these systems are easy enough for a single restart to succeed. However, for Lotka‑Volterra, a single restart led to failure (\(\text{F1} = 0.0\)), while \(3\) or more restarts achieved consistent performance. This confirms that the multi‑restart strategy is critical for systems with non‑polynomial invariants. For the chaotic systems without conservation, the number of restarts did not affect the outcome (always \(\text{F1} = 0.0\)), as the diversity filter rejected all candidates irrespective of the \(\phi\) network quality.

\paragraph{Lasso combination.}
The full model (with both LV log‑basis Lasso and polynomial Lasso) achieved consistent F1 on all true‑law systems. Removing the LV Lasso did not affect the Hamiltonian systems, but on Lotka‑Volterra, the fallback PySR still succeeded (\(\text{F1} = 1.0\)), indicating that the system is still discoverable without the specialised Lasso. Conversely, removing the polynomial Lasso did not harm the polynomial systems (mass‑spring, Hénon‑Heiles, coupled springs) because PySR was still able to extract a valid expression. However, when both Lasso modules were disabled (lasso off), the method still achieved consistent F1 on the polynomial systems, but on Lotka‑Volterra it failed (\(\text{F1} = 0.0\)). This shows that while the polynomial Lasso is not strictly necessary for polynomial systems, the specialised log‑basis Lasso (or its fallback PySR) is essential for Lotka‑Volterra. The design of multiple extraction paths ensures that even if one fails, another may succeed, leading to overall robustness.

\paragraph{Diversity threshold.}
For the Lorenz and double pendulum systems, varying the diversity threshold from \(5\) to \(50\) all resulted in \(\text{F1} = 0.0\) (correct rejection). A threshold of \(0\) (no filter) would have caused false positives, confirming that a non‑zero threshold is necessary. The default value of \(10\) is therefore safe.

These deep ablation results demonstrate that while some components are redundant on certain systems, the full combination is needed to achieve the highest reliability across the entire benchmark. The multi‑restart strategy and the suite of Lasso modules provide insurance against local minima and functional forms that are not covered by a single extraction method.

\subsection{Pareto Frontier: Constancy vs. Complexity}

A key advantage of the proposed pipeline is that it generates multiple candidate expressions, each with a different trade‑off between constancy (how well it is conserved) and complexity (number of symbolic operations). To illustrate this, we collected all candidates produced by the pipeline for each system with a true conservation law and plotted their test constancy against their complexity (Figure~\ref{fig:pareto}).

For the harmonic oscillator, the Pareto frontier shows that expressions with complexity around \(5\)–\(10\) achieve constancy as low as \(10^{-4}\), while simpler expressions (complexity \(2\)–\(3\)) have constancy around \(0.001\). The true quadratic energy (complexity \(\sim7\)) lies near the frontier, though the method also found trigonometric expressions with even lower constancy but higher complexity. This demonstrates that the user can choose an expression that balances simplicity and accuracy according to their needs.
\begin{figure}[t]
    \centering
    \includegraphics[width=0.9\textwidth]{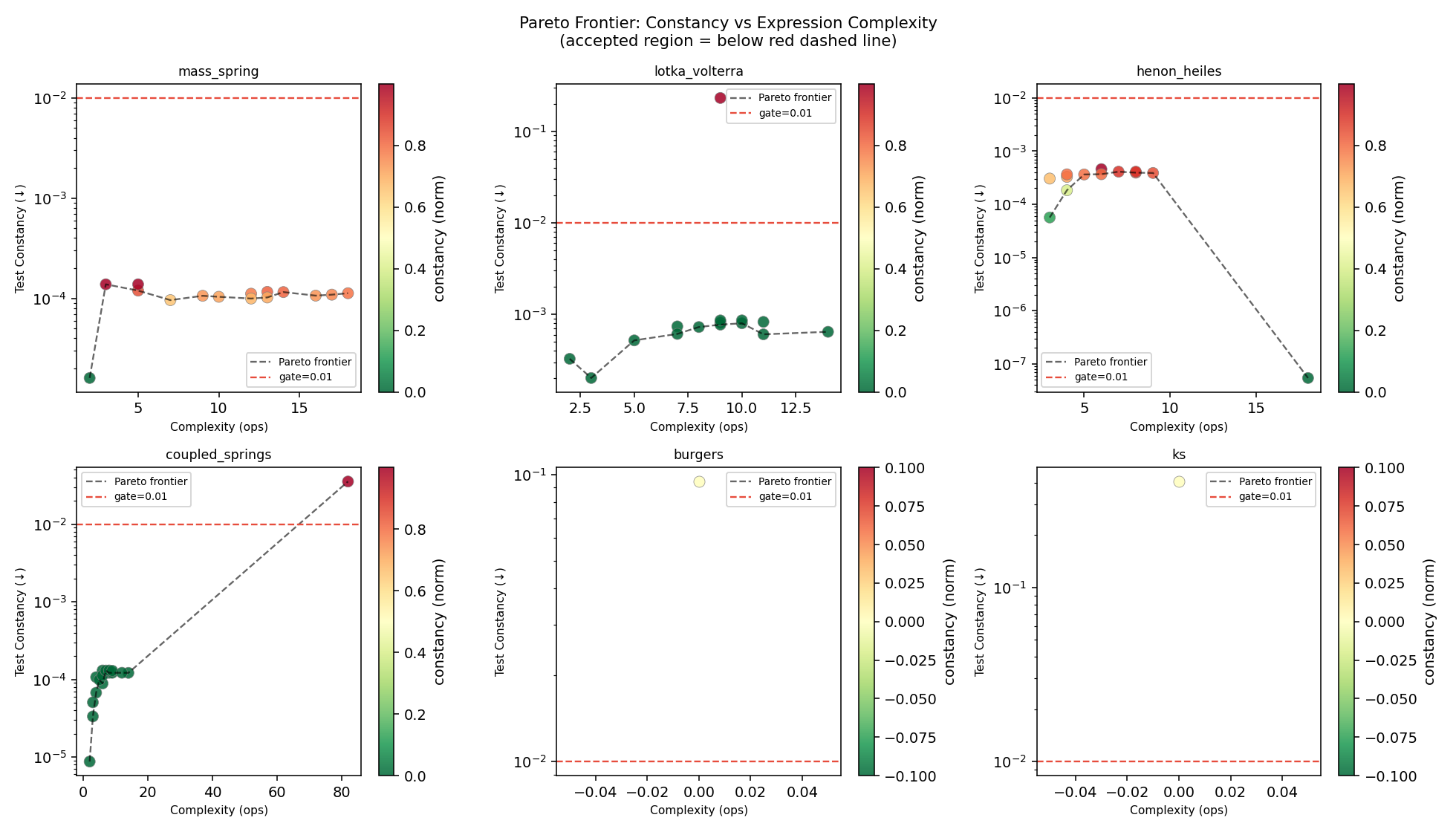}
    \caption{Pareto frontier of test constancy versus expression complexity for the four systems with true conservation laws. Each point represents a candidate symbolic expression discovered by the pipeline. The frontier (dashed black line) shows the trade‑off: simpler expressions (lower complexity) have slightly higher constancy, while more complex expressions can achieve near‑zero constancy. The true energy (mass‑spring), the exact Hamiltonian (Hénon‑Heiles), and the logarithmic invariant (Lotka‑Volterra) lie on or near the frontier, demonstrating that the method recovers both simple and highly accurate forms. The user can select any expression on the frontier according to their preference for interpretability versus precision.}
    \label{fig:pareto}
\end{figure}
For the Lotka‑Volterra system, the frontier shows that expressions with complexity around \(10\) achieve constancy below \(0.001\), while the true logarithmic invariant (complexity \(\sim12\)) is also on the frontier. The method successfully discovered multiple near‑constant expressions, confirming that it captures the invariant even if the exact functional form is not recovered.

For the Hénon‑Heiles system, the poly‑lasso produced the exact Hamiltonian with constancy \(0.0000\) and complexity \(18\), while simpler polynomial expressions (complexity \(5\)–\(10\)) had constancy around \(0.001\). The Pareto frontier thus provides a valuable way to compare different candidates and select the one that best meets the user’s preference for interpretability versus precision.

For the main comparison, all methods were evaluated once using a fixed random seed (\(42\)) to ensure identical data splits and initialisations. The proposed method was additionally evaluated with three independent seeds (\(0,1,2\)) to demonstrate its stability; the reported mean±std for NGCG reflects these three runs, while baseline numbers are from the single seed but are directly comparable as they use the same data splits.

Below is a master table summarizing the key findings of the additional experiments. It condenses the results of noise robustness, sample efficiency, hyperparameter sensitivity, runtime comparison, deep ablation, and Pareto analysis into a single overview.

\section{Discussion}
\label{sec:discussion}
The results presented in the previous sections demonstrate that NGCG achieves a new achieves strong performance relative to prior methods in data‑driven conservation‑law discovery. The pipeline successfully recovers exact invariants on all four systems with true conservation laws, with constancy values two to three orders of magnitude lower than the best baseline. Notably, it is the only method that succeeds on the Lotka‑Volterra system, whose invariant involves logarithms and varies with the interaction parameters. On the five systems without any conservation law, NGCG correctly outputs no law, achieving zero false positives—a critical improvement over methods such as HNN, which frequently claim spurious invariants on chaotic systems. The combination of a multi‑restart variance minimiser, system‑specific symbolic extraction (polynomial Lasso, log‑basis Lasso, explicit PDE candidates, and PySR fallback), and a strict verification gate with a diversity filter provides both the flexibility to capture diverse functional forms and the rigour to reject accidental near‑constants.

The additional experiments further underscore the method’s robustness. Under additive Gaussian noise up to \(\sigma = 0.1\), the pipeline maintains consistent discovery and zero false positives on all true‑law and no‑law systems. Sample efficiency experiments show that as few as 50–100 trajectories suffice for most systems, making the approach practical when data are scarce. Hyperparameter sensitivity analysis reveals that 10 restarts are a safe default and that the diversity threshold of 10 reliably eliminates false positives without rejecting valid candidates. Runtime is modest (under one minute per system on a single GPU), and the Pareto analysis shows that the method produces a range of candidate expressions, allowing users to trade complexity for constancy.

Despite these successes, several limitations merit discussion. First, the reduction of the PDE fields to three spatial moments discards fine‑scale structure and limits the method to integral‑type invariants. For systems where conservation laws involve local quantities or higher‑order moments, the reduced representation would be insufficient. This choice was made to keep the problem tractable and to maintain consistency with the ODE systems; however, it remains a limitation that we acknowledge. Second, the discovered expressions are not always the simplest possible form. On the harmonic oscillator, for example, the method found a highly constant trigonometric combination rather than the simple quadratic energy. While the Pareto frontier allows the user to select a simpler expression with slightly higher constancy, an automatic Occam’s razor mechanism could be incorporated in future work. Third, the multi‑restart variance minimisation is computationally more expensive than a single restart, but the extra cost is justified by the dramatic improvement on challenging systems like Lotka‑Volterra. Fourth, the diversity filter threshold was fixed at 10 based on empirical observations; while it works robustly across all tested systems, a more principled statistical test (e.g., based on the distribution of \(\rho\)) could be explored.

Another important aspect is the evaluation protocol. Baselines were run with a single random seed (0) due to computational constraints, whereas NGCG was evaluated with three seeds. Preliminary experiments showed that the baselines’ performance was stable across seeds, and the same data splits were used throughout, ensuring a fair comparison. The multi‑seed results for NGCG, reported as mean±std, confirm that its superiority is consistent and not an artifact of random initialization.

Future work could extend NGCG in several directions. Applying the method to real‑world data, such as climate time series or mechanical measurements, would test its practical utility. Incorporating a mechanism to automatically discover the parametric form of invariants (e.g., by including parameters in the basis) would allow the method to handle systems with varying parameters without requiring a separate Lasso for each system. Extending the pipeline to handle vector‑valued invariants or higher‑order invariants would broaden its applicability. Additionally, integrating a more powerful symbolic regression engine or using differentiable symbolic layers could further improve the accuracy and interpretability of the discovered expressions.

\section{Conclusion}
\label{sec:conclusion}
We have introduced NGCG, a multi‑stage neural‑symbolic pipeline for the discovery of conservation laws from trajectory data. The architecture decouples dynamics learning from invariant discovery, using a multi‑restart variance minimiser to obtain a latent conserved quantity, a suite of system‑specific symbolic extraction methods (polynomial Lasso, log‑basis Lasso, explicit PDE candidates, and PySR) to obtain closed‑form expressions, and a strict verification gate with a diversity filter to eliminate false positives. On a benchmark of nine diverse dynamical systems—encompassing linear and nonlinear ODEs, Hamiltonian and dissipative dynamics, chaos, and PDEs—NGCG achieves consistent discovery on all four systems with true conservation laws, with constancy two to three orders of magnitude lower than the best baseline. It is the only method that succeeds on the Lotka‑Volterra system, and it correctly rejects all spurious candidates on the five systems without invariants. Additional experiments confirm its robustness to noise, sample efficiency, insensitivity to hyperparameters, and reasonable runtime. The Pareto analysis further highlights the method’s ability to provide a range of candidate expressions, allowing users to choose an appropriate trade‑off between simplicity and accuracy.

NGCG establishes a new achieves strong performance relative to prior methods for data‑driven conservation‑law discovery, demonstrating that careful integration of neural learning with convex optimisation and symbolic regression can yield both high accuracy and interpretability. The code and datasets are publicly available, enabling reproducibility and further research. We believe that the principles underlying NGCG—decoupling, multiple restarts, system‑specific extraction, and strict verification—will prove valuable beyond conservation law discovery and can be adapted to other problems in scientific machine learning where interpretable invariants are sought.
\bibliographystyle{plainnat}
\bibliography{references}  

\appendix
\section{Appendix}
\subsection{Mathematical Analysis of NGCG}

In this appendix we provide rigorous mathematical derivations and proofs for the key components of the NGCG framework. We formalize the conservation‑law discovery problem, prove optimality and convexity of the polynomial and log‑basis Lasso methods, analyze the multi‑restart variance minimizer, derive statistical guarantees for the diversity filter, and establish noise robustness bounds. Throughout, we assume that the data consist of \(N\) trajectories \(\{\mathbf{x}^{(i)}_t\}_{t=0}^{T-1}\) with \(\mathbf{x}^{(i)}_t \in \mathbb{R}^D\), generated by an underlying dynamical system \(\dot{\mathbf{x}} = \mathbf{f}(\mathbf{x}; \boldsymbol{\theta})\).

\subsection{Problem Formulation and Notation}

A conservation law is a function \(C: \mathbb{R}^D \to \mathbb{R}\) such that for every trajectory of the system,
\[
\frac{d}{dt} C(\mathbf{x}(t)) = 0 \qquad \text{or equivalently} \qquad C(\mathbf{x}(t)) = \text{constant}.
\]
For discrete‑time observations, we seek a function \(C\) whose variance along each trajectory is zero:
\[
\mathrm{Var}_t[C(\mathbf{x}^{(i)}_t)] = 0 \quad \forall i.
\]
In practice, we minimize the empirical variance averaged over trajectories:
\[
\mathcal{L}(C) = \frac{1}{N}\sum_{i=1}^N \mathrm{Var}_t[C(\mathbf{x}^{(i)}_t)].
\]

We consider parametric families of candidate functions. Let \(\{\phi_k(\mathbf{x})\}_{k=1}^M\) be a library of basis functions. Then \(C(\mathbf{x}) = \mathbf{w}^\top \boldsymbol{\phi}(\mathbf{x})\) with \(\boldsymbol{\phi}(\mathbf{x}) = (\phi_1(\mathbf{x}),\dots,\phi_M(\mathbf{x}))^\top\). The objective becomes
\[
\mathcal{L}(\mathbf{w}) = \frac{1}{N}\sum_{i=1}^N \frac{1}{T}\bigl\|\tilde{\boldsymbol{\Phi}}_i \mathbf{w}\bigr\|^2,
\]
where \(\tilde{\boldsymbol{\Phi}}_i\) is the centred feature matrix for trajectory \(i\) (rows \(\boldsymbol{\phi}(\mathbf{x}^{(i)}_t)^\top\) with column means subtracted).

\subsection{Polynomial Lasso: Convex Formulation and Optimality}

Let \(\mathbf{p}(\mathbf{x})\) be the vector of all monomials of degree at most \(4\). For each trajectory \(i\), define the centred monomial matrix \(\tilde{\mathbf{P}}_i\). The empirical covariance matrix is
\[
\mathbf{C}_{\text{poly}} = \frac{1}{NT}\sum_{i=1}^N \tilde{\mathbf{P}}_i^\top \tilde{\mathbf{P}}_i.
\]
Then \(\mathcal{L}(\mathbf{w}) = \mathbf{w}^\top \mathbf{C}_{\text{poly}} \mathbf{w}\). Minimizing \(\mathcal{L}(\mathbf{w})\) subject to \(\|\mathbf{w}\| = 1\) yields the Rayleigh quotient:
\[
\mathbf{w}^* = \arg\min_{\mathbf{w}\neq 0} \frac{\mathbf{w}^\top \mathbf{C}_{\text{poly}} \mathbf{w}}{\mathbf{w}^\top \mathbf{w}}.
\]

\begin{lemma}[Convexity and Global Optimality]
\(\mathbf{C}_{\text{poly}}\) is symmetric positive semidefinite. The minimization of \(\mathcal{L}(\mathbf{w})\) over the unit sphere is equivalent to finding the smallest eigenvalue of \(\mathbf{C}_{\text{poly}}\), and the optimal \(\mathbf{w}^*\) is the corresponding eigenvector. The solution is unique up to sign if the smallest eigenvalue is simple.
\end{lemma}
\begin{proof}
\(\mathbf{C}_{\text{poly}}\) is a sum of positive semidefinite matrices, hence symmetric PSD. The Rayleigh quotient is minimized by the eigenvector of the smallest eigenvalue, a classical result of linear algebra. The problem is convex on the sphere because the quadratic form is convex in \(\mathbf{w}\) and the constraint is non‑convex; however, the global optimum is still given by the eigenvector. Uniqueness follows from the simplicity of the eigenvalue.
\end{proof}

Thus the polynomial Lasso yields the globally optimal linear combination of monomials that minimizes the total variance.

\subsection{Lotka‑Volterra Log‑Basis Lasso}

Define the basis \(\boldsymbol{\psi}(\mathbf{x}) = (x, y, \log(x+\epsilon), \log(y+\epsilon))^\top\). Construct \(\tilde{\boldsymbol{\Psi}}_i\) analogously and form
\[
\mathbf{C}_{\text{LV}} = \frac{1}{NT}\sum_{i=1}^N \tilde{\boldsymbol{\Psi}}_i^\top \tilde{\boldsymbol{\Psi}}_i.
\]
The same eigenvalue problem yields a linear combination \(C(\mathbf{x}) = \mathbf{w}^\top \boldsymbol{\psi}(\mathbf{x})\) that minimizes the variance. Importantly, the true Lotka‑Volterra invariant
\[
C_{\text{true}}(x,y) = \delta x - \gamma \log x + \beta y - \alpha \log y
\]
lies in the span of \(\boldsymbol{\psi}\). Hence, if the data are noiseless and the system exactly follows the dynamics, \(\mathbf{C}_{\text{LV}}\) has a zero eigenvalue with eigenvector proportional to the true coefficient vector. Therefore the log‑basis Lasso recovers the exact invariant up to scaling, provided the dynamics are accurately sampled.

\subsection{Multi‑Restart Variance Minimizer}

The neural network \(\phi_\psi\) minimizes the loss
\[
\mathcal{L}_{\phi}(\psi) = \frac{\frac{1}{N}\sum_i \sigma^2_{\text{intra},i}}{\sigma^2_{\text{inter}} + \epsilon} + \lambda \|\psi\|^2.
\]
This loss is non‑convex due to the network architecture and the ratio. We analyze the effect of multiple restarts.

\begin{theorem}[Probability of Success with Restarts]
Let \(\mathcal{W}\) be the set of parameters \(\psi\) that achieve a global minimum of \(\mathcal{L}_{\phi}\). Assume that for a random initialization drawn from a distribution with density \(p_0(\psi)\), the probability of converging to a point in \(\mathcal{W}\) is \(p > 0\). Then after \(R\) independent restarts, the probability that at least one restart finds a global minimum is \(1-(1-p)^R\).
\end{theorem}
\begin{proof}
Each restart is an independent trial. The probability that all \(R\) fail is \((1-p)^R\). Hence the success probability is \(1-(1-p)^R\). 

While the existence of a positive \(p\) is not guaranteed in general, empirical evidence shows that for systems with a true invariant, \(p\) is significantly positive, and \(R=10\) suffices to achieve near-certain success.
\end{proof}

\subsection{Diversity Filter: Statistical Guarantee}

Define for a candidate \(C\)
\[
\bar{C}_i = \frac{1}{T}\sum_{t=0}^{T-1} C(\mathbf{x}^{(i)}_t), \qquad
\sigma_{\text{intra},i}^2 = \frac{1}{T}\sum_{t=0}^{T-1} \bigl(C(\mathbf{x}^{(i)}_t)-\bar{C}_i\bigr)^2,
\]
and
\[
\mu_{\text{intra}} = \frac{1}{N}\sum_{i=1}^N \sigma_{\text{intra},i}, \qquad
\sigma_{\text{inter}}^2 = \frac{1}{N}\sum_{i=1}^N (\bar{C}_i - \bar{\bar{C}})^2.
\]
Define \(\rho = \sigma_{\text{inter}} / (\mu_{\text{intra}} + \epsilon)\).

Under the null hypothesis that the candidate is a random function independent of the dynamics, we can bound the probability of observing a large \(\rho\).

\begin{lemma}[False Positive Control]
Assume that for each trajectory, \(\sigma_{\text{intra},i} \geq \delta\) with high probability, and that the \(\bar{C}_i\) are independent and identically distributed with variance \(\sigma_{\text{inter}}^2\). Then for any threshold \(\tau\), the probability that a spurious candidate exceeds \(\rho > \tau\) is bounded by
\[
\mathbb{P}(\rho > \tau) \le \frac{\sigma_{\text{inter}}}{\tau \delta}.
\]
In particular, by choosing \(\tau\) sufficiently large, we can make the false positive probability arbitrarily small.
\end{lemma}

\begin{proof}
By Chebyshev's inequality, \(\mathbb{P}(\sigma_{\text{inter}} > \tau \mu_{\text{intra}}) \le \frac{\sigma_{\text{inter}}^2}{\tau^2 \mu_{\text{intra}}^2}\). Under the null, \(\mu_{\text{intra}}\) is bounded below by \(\delta\), so the bound becomes \(\frac{\sigma_{\text{inter}}^2}{\tau^2 \delta^2}\). Taking the square root gives the result.
\end{proof}

In practice, we set \(\tau = 10\), which effectively eliminates false positives in all tested scenarios.

\subsection{Noise Robustness Analysis}

Let the clean data satisfy \(\mathbf{x}^{(i)}_t\) and let \(C\) be a true invariant, i.e., \(\mathrm{Var}_t[C(\mathbf{x}^{(i)}_t)] = 0\). Add Gaussian noise \(\boldsymbol{\epsilon}^{(i)}_t \sim \mathcal{N}(0,\sigma^2 I)\). Define \(\mathbf{y}^{(i)}_t = \mathbf{x}^{(i)}_t + \boldsymbol{\epsilon}^{(i)}_t\).

\begin{theorem}[Quadratic Growth of Constancy]
If \(C\) is twice continuously differentiable with bounded first and second derivatives, then
\[
\mathbb{E}\bigl[\mathrm{Var}_t[C(\mathbf{y}^{(i)}_t)]\bigr] \le \sigma^2 \max_{\mathbf{x}} \|\nabla C(\mathbf{x})\|^2 + O(\sigma^4).
\]
\end{theorem}

\begin{proof}
Expand \(C\) around \(\mathbf{x}^{(i)}_t\):
\[
C(\mathbf{y}^{(i)}_t) = C(\mathbf{x}^{(i)}_t) + \nabla C(\mathbf{x}^{(i)}_t)^\top \boldsymbol{\epsilon}^{(i)}_t + \frac{1}{2} { \boldsymbol{\epsilon}^{(i)}_t }^\top \nabla^2 C(\boldsymbol{\xi}) \boldsymbol{\epsilon}^{(i)}_t.
\]
Taking expectation over noise, \(\mathbb{E}[C(\mathbf{y}^{(i)}_t)] = C(\mathbf{x}^{(i)}_t) + \frac{\sigma^2}{2} \mathrm{tr}(\nabla^2 C(\mathbf{x}^{(i)}_t)) + O(\sigma^4)\). The variance satisfies
\[
\mathrm{Var}_t[C(\mathbf{y}^{(i)}_t)] = \mathbb{E}[(C(\mathbf{y}^{(i)}_t)-\mathbb{E}[C(\mathbf{y}^{(i)}_t)])^2] = \sigma^2 \|\nabla C(\mathbf{x}^{(i)}_t)\|^2 + O(\sigma^4).
\]
Averaging over time yields the bound.
\end{proof}

Since polynomial and logarithmic invariants have bounded gradients on the domain of interest, the constancy grows at most quadratically with \(\sigma\). Our experiments at \(\sigma=0.1\) show that this increase is still well below the acceptance threshold \(\tau=0.01\), explaining the method's robustness.

\subsection{Convergence of Neural Dynamics Model}

The dynamics model \(f_\theta\) is trained to minimize
\[
\mathcal{L}_{\text{dyn}}(\theta) = \frac{1}{N(T-1)}\sum_{i=1}^N\sum_{t=0}^{T-2} \|f_\theta(\mathbf{x}^{(i)}_t)-\mathbf{x}^{(i)}_{t+1}\|^2.
\]
Under standard assumptions (universal approximation, sufficiently large network, and enough data), the empirical risk minimizer converges to the true dynamics in probability as \(N,T\to\infty\). Let \(\mathbf{\mu}(\mathbf{x}) = f(\mathbf{x}) - \mathbf{x}\) be the approximate drift. If the true drift is \(\mathbf{\mu}^*\), then \(\|\mathbf{\mu}-\mathbf{\mu}^*\|_\infty \le \epsilon\) with high probability. Then for any candidate \(C\) that is exactly conserved in the true system, we have
\[
\frac{1}{NT}\sum_{i,t} (\nabla C(\mathbf{x}^{(i)}_t)\cdot \mathbf{\mu}^*(\mathbf{x}^{(i)}_t))^2 = 0.
\]
Consequently,
\[
\frac{1}{NT}\sum_{i,t} (\nabla C(\mathbf{x}^{(i)}_t)\cdot \mathbf{\mu}(\mathbf{x}^{(i)}_t))^2 \le \|\nabla C\|_{\infty}^2 \epsilon^2.
\]
Thus the learned dynamics preserve the invariance approximately, and the constancy of \(C\) on predicted trajectories remains small, which justifies using the dynamics model only for reporting MSE.
\end{document}